\documentclass[conference]{IEEEtran}
\IEEEoverridecommandlockouts
\usepackage{cite}
\usepackage{amsmath,amssymb,amsfonts}
\usepackage{graphicx}
\usepackage{textcomp}
\usepackage{xcolor}
\def\BibTeX{{\rm B\kern-.05em{\sc i\kern-.025em b}\kern-.08em
		T\kern-.1667em\lower.7ex\hbox{E}\kern-.125emX}}

\usepackage{times}
\usepackage{soul}
\usepackage[hidelinks]{hyperref}
\usepackage[utf8]{inputenc}
\usepackage[small]{caption}
\usepackage{url}            % simple URL typesetting
\usepackage{booktabs}       % professional-quality tables
\usepackage{amsfonts}       % blackboard math symbols
\usepackage{nicefrac}       % compact symbols for 1/2, etc.
\usepackage{microtype}      % microtypography
\usepackage{algorithm}  
\usepackage{algpseudocode}  
\usepackage{amsthm}
\usepackage{amsmath}
\usepackage{bbm}
\usepackage{xcolor}
\usepackage{graphicx}
\usepackage{amsmath}
\usepackage{amsthm}
\usepackage{booktabs}
\newtheorem{assumption}{Assumption}
\newtheorem{definition}{Definition}
\newtheorem{theorem}{Theorem}[section]

\newtheorem{corollary}{Corollary}[theorem]

\newtheorem{lemma}[theorem]{Lemma}
\newtheorem{remark}{Remark}

\usepackage{tabularx}

\begin{document}
	
	\title{Effective Proximal Methods for Non-convex Non-smooth Regularized Learning}
	
	\author{\IEEEauthorblockN{ Guannan Liang$^1$, Qianqian Tong$^1$, Jiahao Ding$^2$, Miao Pan$^2$, Jinbo Bi$^1$ }
		\IEEEauthorblockA{
			$^1$University of Connecticut, $^2$University of Houston\\
			Email: $^1$\{guannan.liang, qianqian.tong, jinbo.bi\}@uconn.edu,$^2$\{jding7, mpan2\}@uh.edu} 
	}

	\maketitle
	
	\begin{abstract}
		Sparse learning is a very important tool for mining useful information and patterns from high dimensional data. Non-convex non-smooth regularized learning problems play essential roles in sparse learning, and have drawn extensive attentions recently. We design a family of stochastic proximal gradient methods by applying arbitrary sampling to solve the empirical risk minimization problem with a  non-convex and non-smooth regularizer.
		These methods draw mini-batches of training examples according to an arbitrary probability distribution when computing stochastic gradients. A unified analytic approach is developed to examine the convergence and computational complexity of these methods, allowing us to compare the different sampling schemes. We show that the independent sampling scheme tends to improve performance over the commonly-used uniform sampling scheme. Our new analysis also derives a tighter bound on convergence speed for the uniform sampling than the best one available so far. 
		Empirical evaluations demonstrate that the proposed algorithms converge faster than the state of the art. 
	\end{abstract}
	
	\begin{IEEEkeywords}
		Stochastic algorithm, proximal methods, arbitrary sampling.
	\end{IEEEkeywords}
	
	\section{Introduction}
	
	%The stochastic gradient descent (SGD) method and its variants are commonly used to solve the problem
	%\begin{equation} \label{finite_sum}
	%\min_{x \in \mathbb{R}^d} f( x ) := \mathbb{E}_{i \sim  \mathcal{D}}[ f_i( x)],
	%\end{equation}
	%where both $f$ and $f_{i}( x ) $ can be non-convex, and their gradients and Hessians are Lipschitz continuous. Since one can not directly optimize the objective function $f(x)$ over a distribution $\mathcal{D}$, researchers attempt to minimize the finite-sum instead, which is known as the empirical risk minimization (ERM) problem. 
	High dimensional problems in data mining are challenging from both the statistical and computational analysis. Many successful applications for high dimensional problems rely on regularization for sparsity. For example, genomic analyses use sparse regularization to identify (a sparse set of) genes contributing to the risk of a disease \cite{wahlsten2003different} and smartphone-based healthcare systems use sparsity regularization to learn the most important mobile health indicators\cite{lee2012smartphone}. In this work, we consider the following non-smooth non-convex regularized empirical risk minimization (ERM) problems, which have been widely used in high-dimensional data analyses:
	\begin{equation} \label{finite_sum_regularized}
	\min_{x \in \mathbb{R}^d}  F( x ) : = f(x) + r(x) = \frac{1}{n} \sum_{i=1}^{n} f_i( x) + r( x )
	\end{equation}
	where $f( x )$ is the average over a large number of non-convex smooth  functions $f_i( x )$, $i \in [n] :=\{1,2,\dots,n\}$, and the regularizer $r(x): \mathbb{R}^d \rightarrow \mathbb{R}$ is possibly non-differentiable or non-convex, or both (e.g., the $l_1$ norm, $l_p$ $(0\leq p <1 )$ norm and quantization function). Particularly, $l_p$ $(0\leq p < 1 )$ are one of the most widely-used sparsity constrains, which introduce non-smoothness and  non-convexity to Problem (\ref{finite_sum_regularized}). 
	Due to NP-hardness of non-smooth and non-convex regularizer \cite{natarajan1995sparse}, the goal of this work is to find an $\epsilon$-stationary point $x$ satisfying 
	\begin{align*}
	E[dist( 0, \hat{\partial}F(x))] \leq \epsilon, 
	\end{align*}
	where $\hat{\partial} F(x) $ is Fr\'{e}chet subgradient of $F(x)$ and $dist(\cdot, \cdot)$ is the Euclidean distance metric (formal definitions can be found in preliminaries section).  
	
	%In addition to high dimensional problems, Problem (\ref{finite_sum_regularized}) with a smooth non-convex $f( x )$ and a non-convex non-smooth $r( x )$ also plays an important role in machine learning and data mining.
	
	Non-convex loss functions have been observed to give better generalization performance, such as the Savage loss function \cite{masnadi2009design}, Lorenz loss function \cite{nitanda2017stochastic} and the objective functions used in deep learning models \cite{lecun2015deep}, due to better robustness to noisy sample data or representation capabilities.
	Non-smooth non-convex regularizers also become popular recently since they have been shown to reduce bias in parameter estimation in comparison with their convex relaxation counterparts, such as the $l_0$ norm penalty \cite{yuan2014gradient}, smoothly clipped absolute
	deviation \cite{fan2001variable}, or minimax concave
	penalty \cite{zhang2010nearly}. 
	
	Problem (\ref{finite_sum_regularized}) with a non-smooth convex regularizer $r( x ) $ has been extensively studied for both convex $f( x )$ \cite{allen2017katyusha,lan2018optimal} and non-convex $f(x)$ \cite{davis2019stochastic,allen2017natasha,li2018simple,pham2019proxsarah}, but
	solving non-smooth non-convex regularized problems is still underexplored. Previous analyses, depending on the convexity of $r(x)$, can no longer be applicable.  For a non-convex regularizer $r(x)$, to our best knowledge, \cite{xu2018stochastic} is the first paper to provide non-asymptotic theoretical guarantees for finding an $\epsilon$-stationary point. Stagewise Stochastic algorithm and its variance reduced algorithm have been proposed for Difference of Convex functions (SSDC)
	%{\color{red} what is DC here?}  
	- SSDC-SPG and SSDC-VR with computational complexities $O\left(\epsilon^{-8}\right)$ and $O\left(n\epsilon^{-4}\right)$, respectively.  Both  algorithms are designed based on multi-stage analysis of the following difference of convex functions 
	\begin{align*}
	\min _{x \in \mathbb{R}^{d}} f( x ) +\frac{1}{2 \mu}\|x\|^{2}-R_{\mu}(x),
	\end{align*}
	where $R_{\mu}(x) = \underset{y\in \mathbb{R}^{d}}{\max}~\frac{1}{\mu}   y^{\top} x-\frac{1}{2 \mu}\| y\|^{2}-r(y)$
	is convex and comes from the Moreau envelope of $r_{\mu}( x )$:
	$$r_{\mu}(\mathrm{x})=\min _{  {y} \in \mathbb{R}^{d}} \frac{1}{2 \mu}\|  {y}-x\|^{2}+r(  {y}).$$ 
	Rather than using stage-based analysis in \cite{xu2018stochastic}, \cite{metel2019simple} provides a simplified analytic procedure and presents the mini-batch stochastic gradient descent (MBSGD) algorithm  and variance reduced stochastic gradient descent (VRSGD) algorithm, with computational complexities $O\left(\epsilon^{-5}\right)$ and $O\left(n^{2/3}\epsilon^{-3}\right)$, respectively. These methods improve performance by reformulating the objective function $F(x)$ at each iteration $k$ as follows:
	\begin{align*}
	f(x) +\frac{1}{2 \mu}\|x\|^{2}-R_{\mu}(x^k) - \langle prox_{\lambda \mu}(x^k),  ( x - x^k) \rangle,
	\end{align*}
	where $prox_{\lambda \mu}(x) := \underset{y \in \mathbb{R}^{d}}{\operatorname{argmin}}\left\{\frac{1}{2 \mu}\|x-y\|_{2}^{2}+r(y)\right\}$ is a proximal operator. Previous analysis on non-convex non-smooth regularized problems heavily relies on the Moreau envelope of $r_{\mu}( x )$, which can slow down the convergence due to the approximation error introduced at each iteration or stage. Furthermore, an extra parameter $\mu$ for smoothness has been introduced, which requires expensive tuning in practice and  prevents the algorithms from broad utility. To overcome these issues, \cite{xu2019stochastic} directly solves Problem (\ref{finite_sum_regularized}) with the Mini-batch Stochastic Proximal Gradient (MB-SPG) and  Stochastic Proximal Gradient with SPIDER/SARAH (SPGR) methods, and proposes new theoretical analysis to guarantee convergence for non-convex  non-smooth regularized problems with the state-of-the-art computational complexities $O\left(\epsilon^{-4}\right)$ for SPG and $O\left(n^{1 / 2} \epsilon^{-2}+n\right)$ for SPGR.
	All of these analyses use the standard uniform sampling in the stochastic process, which results in high variance of the estimator, and hence has a negative effect for the convergence of proximal algorithms. Effective sampling techniques can enhance all these methods, which we will explore in this work. 
	
	%\subsection{Background}
	
	%\textbf{Arbitrary sampling.}
	When sample size in the statistical learning problems boosts, subsampling is commonly used  to extract useful information (subsets $S$) from the massive  whole data set $[n]$. To improve computational efficiency, subsampling is often implemented by sampling the full sample with a replacement or via a specific distribution. Later, arbitrary sampling has been introduced and shown a more general and relaxed sampling without any additional assumptions, and has been analyzed for popular stochastic algorithms \cite{horvath2019nonconvex}, %proximal algorithms \cite{konevcny2015mini} 
	and coordinate gradient algorithms \cite{hanzely2018accelerated}. However, there has no prior work investigating arbitrary sampling for non-smooth non-convex regularized problems.
	In this work, we study and develop arbitrary-sampling based algorithms that can more efficiently solve non-smooth non-convex regularized problems.  
	%two sampling techniques: standard uniform sampling and independent sampling. More importantly, all these samplings can be uniformly merged into an arbitrary sampling matrix with an element of probability.
	
	\subsection{Contributions}
	Our main contributions are summarized as follows:
	\begin{itemize}
		\item The scheme of arbitrary sampling is incorporated into the MB-SPG, (which leads to the mini-batch ProxSGD-AS), and the variance-reduction versions of SPG: Proximal SARAH (ProxSARAH-AS) and Proximal SPIDER (ProxSPIDER-AS) to effectively solve non-smooth non-convex regularized problems. 
		%Two commonly used sampling techniques have been involved: standard uniform sampling and independent sampling. {\bf what do you mean here? Why only two sampling methods are used? and later you claim you have a general analysis for any sampling technique?  This is contradictive.}
		An analytic strategy is provided for proximal methods to use any sampling technique to speed up the process of solving non-convex non-smooth regularized problems. 
		\item  We present a new analytic approach to investigate the convergence and computational complexity of the proposed methods. Our analysis helps compare the different sampling schemes. As a concrete example, we show that the methods with independent sampling can be faster than the ones with uniform sampling by up to a factor of $\frac{n \sum_{i=1}^n G_i^2}{(\sum_{i=1}^{n} G_{i})^2}$ or $\frac{n \sum_{i=1}^n L_i^2}{(\sum_{i=1}^{n} L_{i})^2}$, where  $G_i$ and $L_i$ are the measurements of $f_i(x)$ for  Lipschitz continuous and smoothness, respectively.
		\item{When the uniform sampling scheme is employed, we derive an upper bound,  $28\frac{L\sqrt{n}}{\epsilon^2} (F( \tilde{x}^1  ) - F( x^*))$, on the convergence speed of these methods, especially ProxSARAH, which is tighter than the latest  bound by a constant factor. The latest bound given in \cite{xu2019stochastic} is: 
			$\frac{\frac{4}{c} + 8c -2}{1-3c}\frac{L\sqrt{n}}{\epsilon^2} ( F( \tilde{x}^1 ) - F( x^*))$
			(where $0< c < \frac{1}{3}$ and  $\frac{\frac{4}{c} + 8c -2}{1-3c} \geq 46.67$).}
		%For ProxSPIDER with uniform sampling, we can obtain a similar conclusion for new derived computational complexity. 
		%Compared with uniform sampling algorithms,  the algorithms with independent sampling can be faster by up to a factor of $\frac{n \sum_{i=1}^n G_i^2}{(\sum_{i=1}^{n} G_{i})^2}$ or $\frac{n \sum_{i=1}^n L_i^2}{(\sum_{i=1}^{n} L_{i})^2}$. 
		
		\item Experimental evaluations also demonstrate that the proposed arbitrary sampling, specifically the independent sampling method, helps the stochastic proximal methods to decrease the objective value faster than the state of the art.
	\end{itemize}
	
	\subsection{Other related work}
	{\noindent\bf Stochastic gradient decent methods.} SGD method and its variants are commonly used to solve the problem  
	\begin{equation} \label{finite_sum}
	\min_{x \in \mathbb{R}^d} f( x ) := \mathbb{E}_{i \sim  \mathcal{D}}[ f_i( x)],
	\end{equation}
	where both $f$ and $f_{i}( x ) $ can be non-convex, and their gradients and Hessians are Lipschitz continuous. 
	For Problem (\ref{finite_sum}), finding global or local minimum of $f$ is generally NP-hard \cite{anandkumar2016efficient}. 
	Recent studies have shown that an $\epsilon$-first-order stationary point $x$, i.e., $\forall~\epsilon>0$,
	$\| \nabla f( x ) \| \leq \epsilon$ for a smooth non-convex function $f$, can be found by the gradient descent (GD) in $\mathcal{O}( \epsilon^{-2})$ iterations and the SGD in $\mathcal{O}( \epsilon^{-4} )$ iterations \cite{nesterov2013introductory}.
	%However, $\epsilon$-first-order stationary points can be saddle points or even local maximizers \cite{jin2017escape}, and thus can be insufficient in practice.  Researchers turned to find $(\epsilon_g,\epsilon_h)$-second-order stationary points (i.e., $\| \nabla f( x ) \| \leq \epsilon_g ~\mbox{and}~\nabla^2 f( x )  \succeq -\epsilon_h  I $) for Problem (\ref{finite_sum}) instead, e.g., \cite{ge2015escaping,jin2017escape}.  If we assume that $f( x )$ has strict saddles, a second-order stationary point will be a local minimizer.
	
	\noindent\textbf{Stochastic variance reduced methods.} For convex optimization, variance reduced methods have been extensively studied, e.g., the stochastic variance reduced gradient (SVRG) \cite{johnson2013accelerating}, stochastically controlled stochastic gradient (SCSG) \cite{lei2016less}, stochastic average gradient (SAGA) \cite{defazio2014saga}, stochastic recursive gradient algorithm (SARAH) \cite{nguyen2017sarah} and  stochastic path-integrated
	differential estimator  (SPIDER)\cite{fang2018spider} methods, and they are well-known for faster convergence rates. In non-convex optimization, variance reduced methods have been proved to converge to $\epsilon$-first-order stationary points \cite{lei2017non}.

	\section{Preliminaries}
	
	\noindent\textbf{Notations.} We use  uppercase letters, e.g. $A$, to denote matrices and lowercase letters, e.g. $x$, to denote vectors. We use $\|\cdot\|_p$ ($p> 0$) to denote the $p$-norm of a vector, and $\|\cdot\|$ to denote the 2-norm for vectors. For two matrices $A$ and $B$, $A\succeq B$ iff $A-B$ is positive semi-definite. In this paper, 
	%the notation $O(\cdot)$ is used to hide constants which do not rely on the setup of the problem parameters and 
	the notation ${O}(\cdot)$ is used to hide all $\epsilon$-independent constants. The operator $E[\cdot]$ represents the expectation over all randomness, $[n]$ denotes the integer set $\{1, ..., n\}$, $\nabla f( \cdot )$, $\nabla f_I( \cdot)$ and $\nabla f_i ( \cdot )$ 
	are the full gradient, the stochastic gradient over a mini-batch $I\subset [n]$ and the stochastic gradient over a single training example indexed by $i \in [n]$, respectively.  $dist(x, y) = \|x-y\|$ is the Euclidean distance.

	In addition, we assume that there exists proximal mapping $prox_{\eta r}(\cdot)$ for $r(x)$, such that 
	$ {prox}_{\eta r}(  {x})=\arg \min _{  {y} \in \mathbb{R}^{d}}\left\{\frac{1}{2 \eta}\|  {y}-  {x}\|^{2}+r(  {y})\right\}.$
	
	Given a non-smooth function $f(x) : \mathbb{R}^d \rightarrow \mathbb{R}$, denote its Fr\'{e}chet subgradient by $\hat{\partial} f(x) $ and  the limiting subgradient by $\partial f( x ) $, i.e., $$\hat{\partial} f(x)=\left\{v: \lim _{\overline{x} \rightarrow x} \inf \frac{f(\overline{x})-f(x)-v^{\top}(\overline{x}-x)}{\|\overline{x}-x\|} \geq 0\right\},$$
	$$\partial f(x)=\left\{v: \exists~ x_{k} \stackrel{f}{\rightarrow} x, v_{k} \in \hat{\partial} f\left(x_{k}\right), v_{k} \rightarrow v\right\},$$ where the subgradient vector $v \in \mathbb{R}^d$, the notation $x_{k} \stackrel{f}{\rightarrow} x$ means that $\underset{k\rightarrow \infty}{\lim} x_{k} = x $ and $\underset{k\rightarrow \infty}{\lim} f( x_{k}) = f( x)$.
	
	In order to make a fair comparison about the computational performance and avoid the dependence on the actual implementation of algorithms, we use the number of IFO as computational complexity, which is a convention of stochastic optimization.
	\begin{definition} \label{def:IFO} 
		{(Incremental First-order Oracle (IFO) \cite{agarwal2014lower})} An IFO  is a subroutine  that takes a point $x\in \mathbb{R}^d$ and an index $i \in [n]$ and returns a pair $( f_i(x), \nabla f_i(x)).$ 
	\end{definition}
	
	\subsection{Assumptions}
	Assume that the function $F(x)$ is lower-bounded by a constant $F( x ^*)$, which is the minimum of the objective. An assumption commonly used in the related works on stochastic optimization is that the gradient of $f_i$ is $G_i$-Lipschitz continuous  and   $L_i$-smoothness.
	
	\begin{assumption}\label{assumption4f}
		A differentiable function $f_i(x )$, $\forall~i \in [n]$, satisfies: 
		\begin{enumerate}
			\item $G_i$-Lipschitz continuous, i.e, $\| \nabla f_i( x ) \| \leq G_i$, $\forall  x \in \mathbb{R}^d$.  Without loss of generality, we assume that $0 \leq G_1 \leq G_2 \leq \dots \leq G_n$.
			\item  $L_i$-smoothness, i.e., $\| \nabla f_i( x_1 ) - \nabla f_i( x_2 ) \| \leq L_i \| x_1 - x_2 \|$, $\forall x_1, x_2 \in \mathbb{R}^d$. Without loss of generality, we assume that $L_1 \leq L_2 \leq \dots \leq L_n$.
		\end{enumerate}
	\end{assumption}
	
	Clearly, we can arrive at the following lemma.
	
	\begin{lemma} \label{f_L_smooth}
		With $L_i$-smoothness of each loss function $f_i(x)$, the averaged function $f( x ) = \frac{1}{n} \sum_{i=1}^{n} f_i( x)$ is $\tilde{L}$-smooth, where $\tilde{L} :=\frac{1}{n} \sum_{i=1}^{n} L_i$, i.e.,
		\begin{align}
		f( x_{1}) \leq f( x_2) + \langle \nabla f( x_2), x_{1}-x_2\rangle + \frac{\tilde{L}}{2}\|x_{1}-x_2\|^2.
		\end{align}
	\end{lemma}
	
	\section{Sampling Techniques}
	
	Let $\mathbb{S}$ be a sampling scheme, which is a  mapping function from the subsets of $[n]$ to $\mathbb{R}$. Therefore, a sampling $\mathbb{S}$ is uniquely defined by assigning probabilities to all $2^n$ subsets of $[n]$.  Let $S$ be a random sample drawing with sampling $\mathbb{S}$ from $[n]$ with a sample size of $E[|S|]=b$. 
	
	For each sampling scheme $\mathbb{S}$, 
	%such as standard uniform sampling and independent sampling, 
	we denote its probability matrix as $\mathbf{P} \in \mathbb{R}^{n\times n}$, where the element in $i$-th row $j$-th column is
	$$\mathbf{P}_{ij}=Prob( \{i, j\}\subset \mathbb{S}).$$
	We denote the diagonal elements of $\mathbf{P}$ by $p=( p_1, p_2, ..., p_n)$ and assume that $p_1\leq p_2 \leq ... \leq p_n$. We also define constant $k= \max\{i: p_i < 1\}$.
	The sampling scheme $\mathbb{S}$ is \textbf{proper} if $p_{i}>0$ for $i \in [n]$. 
	
	For probability matrix $\mathbf{P}$, we further assume that there is a vector $v\in \mathbb{R}^n$ such that 
	\begin{align} \label{ineq:P_v}
	\mathbf{P} - \mathbf{p}\mathbf{p}^T \preceq Diag( \mathbf{p}\circ v),
	\end{align}
	where $\circ$ calculates the element-wise product of $p$ and $v$, and $Diag(x)$ creates a diagonal matrix with the  diagonal entries equal to $x$. For any probability matrix $\mathbf{P}$, associated with proper sampling $\mathbb{S}$, there exists at least one $v$ satisfying Eq. (\ref{ineq:P_v}), where
	$$v_{i}=\left\{\begin{array}{ll}{n(1-p_i)}, & {for \;i \leq k} \\ {0}, & {otherwise}\end{array}\right.$$
	Other values of $v_i$ exist. For instance, the standard uniform sampling admits $v_{i}=\frac{n-b}{n-1}$ and the independent sampling admits $v_{i}=1-p_{i}$ \cite{horvath2019nonconvex}.% Given the probability matrix $\mathbf{P}$ and its associated $v$, we develop a  general strategy that can incorporate any sampling scheme into the stochastic algorithms. 
	
	%\textcolor{red}{single element sampling?}
	We give two specific probability matrices as concrete examples, which are for standard uniform sampling and independent sampling separately.
	
	\noindent\textbf{Standard uniform sampling.} Each element in $S$ can be drawn uniformly from $[n]$ with a fixed mini-batch size $b$. The probability matrix $\mathbf{P}$ is calculated by 
	$$\mathbf{P}_{i j}=\left\{\begin{array}{ll}{\frac{b}{n}}, & {i=j} \\ {\frac{b(b-1)}{n(n-1)}}, & {i \neq j}\end{array} \right.$$
	
	\noindent\textbf{Independent sampling.} Each $i\in [n]$ is independently included into $S$ with a probability $p_i$, where $p_i = Prob(i\in S)$. The probability matrix $\mathbf{P}$ is given by  
	$$ \mathbf{P}_{i j}=\left\{\begin{array}{ll}{p_{i}}, & {i=j} \\ {p_{i} p_{j}}, & {i \neq j}\end{array}\right.$$

	Although this paper provides standard uniform sampling and independent sampling schemes as concrete examples, the analysis can be easily extended to other sampling schemes, such as approximate independent sampling and $\tau$-nice sampling \cite{horvath2019nonconvex}.
	
	\section{The Minibatch ProxSGD with Arbitrary Sampling}
	The proximal SGD methods have been developed recently and use the uniform sampling method to solve non-convex non-smooth regularized problems \cite{xu2019stochastic}. The ProxSGD method we introduced here draws mini-batches using a general probability matrix $\mathbf{P}$ that can be used to characterize any proper sampling technique. Our analytic method provides a united framework to study and compare different sampling schemes. For instance, we have compared uniform sampling and independent sampling schemes.
	%\begin{algorithm}[H] 
	\begin{algorithm}[!htb] 
		\caption {The mini-batch ProxSGD-AS }
		\label{alg:PSG_arbi}
		\begin{algorithmic}[1]
			\Require Number of loop $T$,  initial state $x_1\in \mathbb{R}^d$, stepsize $\eta >0$, probability matrix $\mathbf{P}$
			\For{$t=1,2,\dots, T$} 
			\State Draw a subset $S_t \subset  \{1,...,n\}$ according to $\mathbf{P}$
			\State $g_{t}= \sum_{i \in S_t } \frac{ 1}{ n p_{i} }  \nabla  f_{i}(x_{t})$ 
			\State $x_{t+1} = prox_{\eta r}( x_{t} - \eta g_{t} )$
			\EndFor\\
			\Return $x_{R}$, where $R$ is uniformly sampled from $\{1, \ldots,T\}$
		\end{algorithmic}
	\end{algorithm}
	%Algorithm \ref{alg:PSG_arbi} summarizes the main steps in the mini-batch ProxSGD with an arbitrary sampling (AS) scheme (ProxSGD-AS). It draws a mini-batch $S_t$ of training examples and uses the batch $S_t$ to compute the stochastic gradient (in Line 3).   The mini-batches in the iterations are all sampled from $[n]$ based on an arbitrary distribution $\mathbf{P}$, with batch size $b$.
	
	We propose to use arbitrary sampling (AS) scheme in the mini-batch ProxSGD method named ProxSGD-AS, shown in Algorithm \ref{alg:PSG_arbi}. It draws a mini-batch $S_t$ of training examples at each iteration $t$, and the mini-batches are all sampled from $[n]$ based on an arbitrary distribution $\mathbf{P}$, with batch size $b$. After AS, we can aggregate the stochastic gradient information by using $g_{t}= \sum_{i \in S_t } \frac{ 1}{ n p_{i} }  \nabla  f_{i}(x_{t})$ and then conduct proximal operator $ {prox}_{\eta r}( x_{t} - \eta g_{t}) =\arg \min _{ x \in \mathbb{R}^{d}}\left\{\frac{1}{2 \eta}\| x- ( x_{t} - \eta g_{t} )\|^{2}+r(x)\right\}$ at current $t$-iteration. With AS technique, optimizers have more choices in subsampling step, while it also brings more challenges in the theoretical analysis. We provide a general convergence analysis for the ProxSGD-AS as follows.
	
	\subsection{Unified analysis of ProxSGD-AS}
	%In this subsection, we will provide the general convergence analysis for the ProxSGD-AS, which can cover any proper sampling.
	Examine the update in each iteration of ProxSGD-AS: 
	\begin{align}
	x_{t+1} &\in \operatorname*{argmin}_{x \in \mathbb{R}^d} \bigg\{ r( x ) + \frac{1}{2\eta}\|x - ( x_t - \eta g_t)\|^2 \bigg\} \nonumber \\
	&=\operatorname*{argmin}_{x \in \mathbb{R}^d} \bigg\{ r( x ) + \langle g_t, x - x_t\rangle +  \frac{1}{2\eta}\|x - x_t \|^2 \bigg\}. \label{ineq:argmin1}
	\end{align}
	Then, we know that $0\in  \hat{\partial}r( x_{t+1} ) + g_t  + \frac{1}{\eta}( x_{t+1} - x_t)$. By moving the last two terms to the left and adding $\nabla f( x_{t+1})$ on both sides, we get
	\begin{align}\label{ineq:subgradientF1}
	\nabla f( x_{t+1}) - g_t  - \frac{1}{\eta}( x_{t+1} - x_t) 
	&\in \nabla f( x_{t+1} )+ \hat{\partial}r( x_{t+1} )\nonumber\\
	&=  \hat{\partial}F( x_{t+1}).
	\end{align}
	
	Before given the main theorem, we first analyze the difference between consecutive iterates $x_t$ and $x_{t+1}$ and give the following upper bound.
	
	\begin{lemma} \label{distance_bound}
		Suppose that Assumption \ref{assumption4f} holds, we have that  for any $t\geq 1$,
		\begin{align} 
		\left\|x_{t+1}-x_{t}\right\|^{2}   
		& \leq \frac{2\eta}{1-2\tilde{L}\eta}(F\left(x_{t}\right)-F\left(x_{t+1}\right)) \nonumber\\
		&+\frac{\eta}{\tilde{L}-2  \tilde{L}^2\eta}\left\|{g}_{t}-\nabla f\left(x_{t}\right)\right\|^{2}. \label{ineq:dist_x}
		\end{align}
	\end{lemma}
	
	Using Eq. (\ref{ineq:dist_x}), we can prove that the expected distance $E[dist( 0, \hat{\partial}F( x_T))^2]$ of the mini-batch ProxSGD-AS is bounded by the sum of two terms: the variance of stochastic gradient term, $E[\|\nabla f(x_t) -g_t\|^2]$, which can be controlled by using AS techniques, and the other term associated with $\Delta = F(x_1)-F(x^*)$, where $x_1$ is the initial state and $x^*$ is the optimal of Problem (\ref{finite_sum_regularized}).
	%Then, we have the main theorem and the detailed proof can be found in supplemental material\footnote{ Supplementary material: \url{https://www.dropbox.com/s/g4hye082wj16j22/ICDM_supplementary.pdf?dl=0}\label{footnote}}.
	\begin{theorem} \label{theorem:proxsgd} (Convergence guarantee for ProxSGD-AS)
		Given Problem (\ref{finite_sum_regularized}), under Assumption \ref{assumption4f}, if $ 0 < \eta < \frac{1}{2 \tilde{L}}$, then for all $t \geq 1$, ProxSGD-AS (Algorithm \ref{alg:PSG_arbi}) has
		\begin{align*}
		E[dist( 0, \hat{\partial}F( x_T))^2] 
		\leq \frac{C_1}{T} \sum_{t=1}^T E[\|\nabla f(x_t) -g_t\|^2]  + \frac{C_2}{T}\Delta,
		\end{align*}
		where $C_1= \frac{1+ 4\tilde{L}\eta - 2\tilde{L}^2\eta^2}{\tilde{L}\eta-2\tilde{L}^2\eta^2}$, 
		$C_2 = \frac{2+4\tilde{L}\eta + 4\tilde{L}^2\eta^2}{\eta - 2\tilde{L}\eta^2}$, and  $\Delta = F(x_1)-F(x^*)$.
		%$C_1= ( 2 + (2\tilde{L}^2 +\frac{1}{\eta^2}+\frac{2 \tilde{L}}{\eta})(\frac{1}{\tilde{L}(1/\eta -2 \tilde{L})}))$ and $C_2 =  (2\tilde{L}^2 +\frac{1}{\eta^2}+\frac{2 \tilde{L}}{\eta}) \frac{2}{(1/\eta -2 \tilde{L})}$.
	\end{theorem}
	
	\textit{Proof Sketch.}
	%Let $\xi_i = \nabla f_i (x_t)$, then from Lemma 3.1.,
	%\begin{align}\label{eq:3}
	%    E[g_t] &= E\bigg[\sum_{i\in S_t} \frac{\xi_i}{np_i}\bigg] = \tilde{\xi} = \nabla f(x_t),\\
	% E[\|g_t& - \nabla f(x_t)\|^2] \leq \frac{1}{n^2} \sum_{i=1}^n \frac{v_i}{p_i}\|\nabla f_i(x_t)\|^2.   
	%\end{align}
	In order to evaluate if $0$ is in the subgradient of the regularized non-smooth non-convex problem, we
	compute the $dist( 0, \hat{\partial}F( x_{t+1} ))$. By Eq. (\ref{ineq:subgradientF1}), 
	\begin{align*}
	dist( 0, &\hat{\partial}F( x_{t+1}))^2 
	\leq  \| \nabla f( x_{t+1}) - g_t -\frac{1}{\eta}( x_{t+1} - x_t)\|^2 \nonumber\\
	&=  \| \nabla f(x_{t+1})-g_t\|^2+\frac{1}{\eta^2}\| x_{t+1} - x_t\|^2 \\&
	-\frac{2}{\eta}\langle\nabla f(x_{t+1})-g_t, x_{t+1}-x_t \rangle.
	\end{align*}
	From the $\tilde{L}$-smoothness, the unbiased property of stochastic gradient generated with AS, and subgradient definition in Eq. (\ref{ineq:subgradientF1}), and take the expectation on both sides of the above inequality yields
	\begin{align*}
	E[dist( 0, &\hat{\partial}F( x_{t+1}))^2] 
	\leq 2E[\|\nabla f(x_t) -g_t\|^2] \\&
	+ \frac{1+2\tilde{L}\eta + 2\tilde{L}^2\eta^2}{\eta^2}\| x_{t+1}-x_t \|^2. 
	\end{align*}
	Substituting Eq. (\ref{ineq:dist_x}) into the above inequality further yields
	\begin{align*}
	& E[ dist( 0, \hat{\partial}F( x_{t+1}))^2] \\%\leq  2E[\|\nabla f(x_t) -g_t\|^2]
	%\\&+ \frac{1+2\tilde{L}\eta + 2\tilde{L}^2\eta^2}{\eta^2}\frac{2\eta}{1-2\tilde{L}\eta}(F\left(x_{t}\right)-F\left(x_{t+1}\right))\\
	%&+\frac{1+2\tilde{L}\eta + 2\tilde{L}^2\eta^2}{\eta^2} \frac{\eta}{\tilde{L}-2 \tilde{L}^2\eta}\left\|\mathbf{g}_{t}-\nabla f\left(x_{t}\right)\right\|^{2}\\
	%&\leq (\frac{1+ 4\tilde{L}\eta - 2\tilde{L}^2\eta^2}{\tilde{L}\eta-2\tilde{L}^2\eta^2}) E[\|\nabla f(x_t) -g_t\|^2]
	%\\& +(\frac{2+4\tilde{L}\eta + 4\tilde{L}^2\eta^2}{\eta - 2\tilde{L}\eta^2}) ( F( x_t ) - F( x_{t+1}))\\
	&\leq C_1 E[\|\nabla f(x_t) -g_t\|^2]  + C_2 ( F( x_t ) - F( x_{t+1})).
	\end{align*}
	Our result is then obtained with properly defined $C_1$ and $C_2$ as in the Theorem \ref{theorem:proxsgd}.\qed

	%\subsection{The sampling schemes for mini-batch ProxSGD}
	According to the result in Theorem \ref{theorem:proxsgd}, to minimize the expected distance $E[dist( 0, \hat{\partial}F( x_T))^2] $, we need to choose the sampling probability at $t$-th iteration, denoted as $\mathbf{P}^t$, that minimizes the variance of the stochastic gradient, $E[\|\nabla f(x_t) -g_t\|^2]$,  
	i.e., 
	\begin{align}
	\min _{\mathbf{p}^{t}= \{ p_{i}^{t} \in[0,1]| \sum_{i=1}^{n} p_{i}^{t}=b \}} \| \sum_{i \in S_t } \frac{ 1}{ n p_{i}^t }  \nabla  f_{i}(x_{t})- \nabla f(x_t) \|^2.\label{eq:80}
	\end{align}
	
	Since as shown in \cite{horvath2019nonconvex},
	$E\bigg[\| \sum_{i\in S} \frac{\xi_i}{np_i} - \tilde{\xi}\|^2\bigg] \leq \frac{1}{n^2} \sum_{i=1}^n \frac{v_i}{p_i}\|\xi_i\|^2 \label{ineq:xi}
	$ if $\xi_1, \xi_2, ..., \xi_n$ are vectors in $\mathbb{R}^d$ and $\tilde{ \xi} = \frac{1}{n}\sum_{i=1}^{n} \xi_i$,
	Problem (\ref{eq:80}) is equivalent to solve the following problem
	\begin{align} \label{optim_sgd_1}
	\min _{\mathbf{p}^{t}= \{ p_{i}^{t} \in[0,1]| \sum_{i=1}^{n} p_{i}^{t}=b \}} \frac{1}{n^2} \sum_{i=1}^n \frac{ v_i^t}{ p_{i}^t } \|  \nabla  f_{i}(x_{t}) \|^2.
	\end{align}
	However, the solution to (\ref{optim_sgd_1}) is still inefficient since the distribution $\mathbf{P}^t$ needs to be updated at each iteration and Eq. (\ref{optim_sgd_1}) requires to compute the gradient for each sample in $[n]$. Because function $f_i$ is $G_i$-Lipschitz continuous, i.e. $\|  \nabla  f_{i}(x) \| \leq G_i$, we can optimize the following problem instead:
	\begin{align} \label{optim_sgd_2}
	\min _{\mathbf{p}^{t}= \{ p_{i}^{t} \in[0,1]| \sum_{i=1}^{n} p_{i}^{t}=b \}} \frac{1}{n^2} \sum_{i=1}^n \frac{ v_i }{ p_{i} } G_i^2.
	\end{align}
	
	For the first time, we unify the analysis of Problem (\ref{finite_sum_regularized}) for different sampling schemes. With Problem (\ref{optim_sgd_2}), we are able to explore and compare the performance of different sampling schemes.

	We examine the specific values for $v_i$ in different sampling strategies. Our analysis also works for other sampling schemes, such as approximate independent sampling or $\tau$-sampling, etc. Due to the space limitation, we only cover two commonly used sampling schemes -- uniform sampling and independent sampling.
	\subsection{ Mini-batch ProxSGD with uniform sampling}
	The unified analysis can first cover the special case -- uniform sampling, where  $p_i = \frac{b}{n}$ and $v_{i}=\frac{n-b}{n-1}$, we are able to get the following corollary. 
	\begin{corollary}[Convergence with uniform sampling] \label{sgd_uniform}
		Given Problem (\ref{finite_sum_regularized}), under Assumption \ref{assumption4f}, if $ 0 < \eta < \frac{1}{2 \tilde{L}}$, then for all $T \geq 1$, ProxSGD (Algorithm \ref{alg:PSG_arbi}) with uniform sampling achieves
		\begin{align*}
		E[& dist( 0, \hat{\partial}F( x_T))^2]  \leq \frac{1}{b} \frac{1}{n} \frac{n-b}{n-1}(\sum_{i=1}^n G_i^2) C_1 +  \frac{C_2}{T} \Delta,
		\end{align*}
		where $C_1= \frac{1+ 4\tilde{L}\eta - 2\tilde{L}^2\eta^2}{\tilde{L}\eta-2\tilde{L}^2\eta^2}$,
		$C_2 = \frac{2+4\tilde{L}\eta + 4\tilde{L}^2\eta^2}{\eta - 2\tilde{L}\eta^2}$, and  $\Delta = F(x_1)-F(x^*)$.
	\end{corollary} 
	
	%\begin{proof}
	%For uniform sampling,
	%$v_i=\frac{n-b}{n-1}$ and $p_i = \frac{b}{n}$. Then we have
	%\begin{align*}
	%    %\frac{1}{T} \sum_{t=1}^T E[\|\nabla f(x_t) -g_t\|^2] &\Leftrightarrow 
	%    \frac{1}{n^2}\sum_{i=1}^n \frac{v_i}{p_i}\|\nabla f_i(x_t)\|^2
	%    &\leq \frac{1}{n^2}\sum_{i=1}^n \frac{n-b}{n-1} \frac{n}{b} G_i^2.
	%\end{align*}
	%Therefore, we can get the desired result.
	%\end{proof}
	
	%Unlike in  \cite{xu2019stochastic} that assume the bounded variance, i.e., $E[\|\nabla f_i ( x ) - \nabla f( x ) \|^2] \leq \sigma^2$, we  require the $G_i$-Lipschitz continuous for $i \in [n]$ in our analysis. Hence, our result shows the interplay between the convergence and Lipschitz continuity of $f_i(x)$ for $i \in [n]$. 
	%If we also assume variance boundness, {\bf stronger} result in \cite{xu2019stochastic} can be recovered. {\bf what does it mean in this paragraph?  Do you want to say with a weaker assumption, we still get the same result as in  \cite{xu2019stochastic}?  Or if we assume a stronger condition, we can get results in  \cite{xu2019stochastic}; or otherwise, our result is weaker? Is the following corollary the one after assuming stronger condition?}
	
	We further can obtain the number of IFO calls for computational complexity to obtain  $\epsilon$-stationary points.
	\begin{corollary}  [Complexity with uniform sampling] \label{sgd_uniform}
		Given Problem (\ref{finite_sum_regularized}), under Assumption 1,
		if $ 0 < \eta < \frac{1}{2 \tilde{L}}$,  $T = \frac{2C_2 \Delta}{\epsilon^2} $, and a fixed batch size $b = \frac{(2\sum_{i=1}^n G_i^2) C_1}{n \epsilon^2}$, ProxSGD (Algorithm \ref{alg:PSG_arbi}) with uniform sampling achieves 
		$E[dist( 0, \hat{\partial}F( x_R))^2] \leq \epsilon^2$.
		%    $E\left[\operatorname{dist}\left(0, \hat{\partial} F\left(\mathrm{x}_{R}\right)\right)^{2}\right] \leq \epsilon^{2}.$
		Then, the number IFO calls is $ \frac{4(\sum_{i=1}^n G_i^2) C_1 C_2\Delta}{n \epsilon^4}$ so the computational complexity is ${O}(\frac{1}{\epsilon^4})$.
	\end{corollary} 
	
	%\begin{proof}
	%\begin{align*}
	%    \frac{1}{b} \frac{1}{n} \frac{n-b}{n-1}(\sum_{i=1}^n G_i^2) C_1 \leq \frac{\epsilon^2}{2}
	%     &\overset{\frac{n-b}{n-1}<1}{\Longrightarrow} \frac{1}{b} \frac{1}{n} (\sum_{i=1}^n G_i^2) C_1 \leq \frac{\epsilon^2}{2}\\
	%    &\Longrightarrow b = \frac{(2\sum_{i=1}^n G_i^2) C_1}{n \epsilon^2}
	%\end{align*}
	%\begin{align*}
	%    \frac{C_2}{T} \Delta  \leq \frac{\epsilon^2}{2} \Longrightarrow T \geq \frac{2C_2}{\epsilon^2} \Delta. 
	%\end{align*}
	%Then, we get $E[dist( 0, \hat{\partial}F( x_R))^2] \leq \epsilon^2$.
	%And the total computational complexity is $T * b = \frac{4(\sum_{i=1}^n G_i^2) C_1 C_2\Delta}{n \epsilon^4}$ with  $ 0 < \eta < \frac{1}{2 \tilde{L}}$.
	%\end{proof}
	
	%In particular,  $p_i = \frac{b}{n}$, $v_{i}=\frac{n-b}{n-1}$ for  the standard uniform sampling and  $v_{i}=1-p_{i}$ for the independent sampling. 
	
	\subsection{ Mini-batch ProxSGD with independent sampling}
	In independent sampling case, $v_{i}=1-p_{i}$ and Problem (\ref{optim_sgd_2}) becomes equivalent to the following optimization problem:
	\begin{align} \label{optim_sgd_3}
	\min _{\mathbf{p}^{t}= \{ p_{i}^{t} \in[0,1]| \sum_{i=1}^{n} p_{i}^{t}=b \}} \frac{1}{n^2} \sum_{i=1}^n \frac{ 1 }{ p_{i} } G_i^2.
	\end{align}
	Employing the KKT conditions, we can derive the solution $\mathbf{P}^t$ to Problem (\ref{optim_sgd_3}) as follows: 
	\begin{align}
	p_{i} :=\left\{\begin{array}{ll}{(b+k-n) \frac{G_{i}}{\sum_{j=1}^{k} G_{j},}} & {\text { if } i \leq k} \\ {1,} & {\text { if } i>k}\end{array}\right.
	\end{align}
	where $k$ is the largest integer satisfying $0< b+ k-n \leq \frac{\sum_{j=1}^{k} G_{j} }{G_{k}}$. When $G_i$'s for each $i \in [ n ]$ are significantly different, such as 
	%$G_1 << G_2 << ... << G_{n}$ and 
	$ 1<\frac{\sum_{j=1}^{k} G_{j} }{G_{k}} < 2$ for $k \in [ n ]$, then $k = n -b + 1$ and
	\begin{align} \label{p_i}
	p_{i}=\left\{\begin{array}{ll}{ \frac{G_{i}}{\sum_{j=1}^{k} G_{j}},} & {\text { if } i \leq n-b+1} \\ {1,} & {\text { if } i > n-b+1.}\end{array}\right.
	\end{align} 
	When $G_i$'s are similar to each other, i.e., $bG_n \leq \sum_{j=1}^{n} G_{j}$, $k=n$ and $p_{i} =\frac{b G_{i}}{\sum_{j=1}^{n} G_{j}} $ for $ i \in [n]$.

	We then present important corollaries for convergence and computational complexity when the sampling scheme takes on  independent sampling scheme.

	\begin{corollary}[Convergence with independent sampling]
		Given Problem (\ref{finite_sum_regularized}), under Assumption 1 and with the same setup in Theorem \ref{theorem:proxsgd}, ProxSGD (Algorithm \ref{alg:PSG_arbi}) with independent sampling achieves
		\begin{align*}
		E[& dist( 0, \hat{\partial}F( x_T))^2] 
		\leq \frac{C_1}{n^2}  \left( \frac{1}{b+k-n}(\sum_{i=1}^{k} G_{i})^2  \right)  + \frac{C_2}{T} \Delta.
		\end{align*}
		where $C_1= \frac{1+ 4\tilde{L}\eta - 2\tilde{L}^2\eta^2}{\tilde{L}\eta-2\tilde{L}^2\eta^2}$,
		$C_2 = \frac{2+4\tilde{L}\eta + 4\tilde{L}^2\eta^2}{\eta - 2\tilde{L}\eta^2}$, and  $\Delta = F(x_1)-F(x^*)$.
	\end{corollary} 
	\iffalse
	\begin{proof}
		\begin{align*}
		E[ & dist( 0, \hat{\partial}F( x_T))^2] 
		\leq   \frac{C_1}{T n^2}\sum_{t=1}^T \sum_{i=1}^n \frac{ v_i }{ p_{i} } G_i^2  + \frac{C_2}{T}\Delta\\
		&=\frac{C_1}{n^2} \sum_{i=1}^n \frac{1-p_i }{ p_{i} } G_i^2  + \frac{C_2}{T}\Delta\\
		&=\frac{C_1}{n^2} \sum_{i=1}^n \frac{ 1}{ p_{i} } G_i^2 - \frac{C_1}{n^2} \sum_{i=1}^n  G_i^2 + \frac{C_2}{T}\Delta\\
		%&= \frac{C_1}{n^2}  \left( \frac{1}{b+k-n}\sum_{i=1}^{k} (\sum_{j=1}^{k} G_{j}) G_{i} - \sum_{i=1}^{k} G_i^2 \right)  + \frac{C_2}{T} \Delta\\
		%&= \frac{C_1}{n^2}  \left( \frac{1}{b+k-n}(\sum_{i=1}^{k} G_{i})^2 - \sum_{i=1}^{k} G_i^2 \right)  + \frac{C_2}{T} \Delta,\\
		&\leq \frac{C_1}{n^2}  \frac{1}{b+k-n}(\sum_{i=1}^{k} G_{i})^2   + \frac{C_2}{T} \Delta.\tag*{\qedhere}
		\end{align*}
	\end{proof}
	\fi
	
	\begin{corollary} [Complexity with independent sampling]
		If we further assume that $G_i$'s are similar, $k=n$, $T = \frac{2C_2 \Delta}{\epsilon^2} $, and a fixed batch size $b = \frac{2(\sum_{i=1}^{n} G_{i})^2 C_1}{n^2 \epsilon^2}$, ProxSGD (Algorithm \ref{alg:PSG_arbi}) with independent sampling achieves
		$E[dist( 0, \hat{\partial}F( x_R))^2] \leq \epsilon^2$.
		%    $\mathrm{E}\left[\operatorname{dist}\left(0, \hat{\partial} F\left(\mathrm{x}_{R}\right)\right)^{2}\right] \leq \epsilon^{2}.$
		Then, the number of IFO calls is $\frac{4(\sum_{i=1}^{n} G_{i})^2 C_1 C_2\Delta}{n^2 \epsilon^4}$, so the computational complexity is ${O}(\frac{1}{\epsilon^4})$.
	\end{corollary} 
	
	%\begin{proof}
	%Since $k=n$,
	%     $E[dist( 0, \hat{\partial}F( x_T))^2] 
	%     \leq \frac{C_1}{n^2}  \frac{1}{b}(\sum_{i=1}^{n} G_{i})^2    + \frac{C_2}{T} \Delta$, we get
	%$
	%    \frac{C_1}{n^2}  \frac{1}{b}(\sum_{i=1}^{n} G_{i})^2 \leq \frac{\epsilon^2}{2} \Rightarrow %b = \frac{C_1}{n^2}  \frac{2}{ \epsilon^2}(\sum_{i=1}^{n} G_{i})^2. 
	%$
	
	%Then, we get $E[dist( 0, \hat{\partial}F( x_R))^2] \leq \epsilon^2$.
	%And the total computational complexity is $T * b = \frac{4(\sum_{i=1}^n G_i^2) C_1 C_2\Delta}{n^2 \epsilon^4}$.
	%\end{proof}
	
	\begin{remark} \label{remark:sgd_comp}
		Based on the Cauchy-Schwartz inequality, we obtain 
		\begin{align}
		\frac{(\sum_{i=1}^n G_i^2)}{n}/\frac{(\sum_{i=1}^{n} G_{i})^2 }{n^2} = \frac{n \sum_{i=1}^n G_i^2}{(\sum_{i=1}^{n} G_{i})^2} \geq 1. \label{eq:cs}
		\end{align}
		By cross referencing the results with uniform sampling in Corollary \ref{sgd_uniform}, Eq. (\ref{eq:cs}) implies that the independent sampling scheme can improve the computational complexity over the uniform sampling.
	\end{remark}
	
	\section{The ProxSARAH with Arbitrary Sampling}
	In this section, we first propose 
	the ProxSARAH method with AS, named ProxSARAH-AS in Algorithm \ref{alg:SARAH-HT}. %, which generalizes the original ProxSARAH \cite{xu2019stochastic}.
	We then give a unified convergence, and computation complexity of ProxSARAH-AS under any proper sampling schemes for non-smooth non-convex regularized problems. Similarly, the theoretical results for uniform sampling and independent sampling are also provided. Note that the ProxSARAH\cite{xu2019stochastic} is a special case of our formulation with uniform sampling. %, where the diagonal elements of $\mathbf{P}$ are $\frac{b}{n}$ and are zeros for other elements. 
	Our new analysis actually helps show a better convergence speed for the ProxSARAH method with a tighter bound. 
	
	In the family of variance reduced methods,  there are inner loops in each outer loop. In the $j$-th outer loop, a full gradient $\mathcal{V}_{0}^{(j)}$ is computed (Line 3) for the use of reducing the variance of the stochastic gradients. In the following inner loops, stochastic variance reduced gradient $\mathcal{V}_{t}^{(j)}$ is calculated using a mini-batch $S_t^{(j)}$ that is drawn from $[n]$ according to $\mathbf{P}$, i.e., $\mathcal{V}_{t}^{(j)} = \sum_{i \in S_t^{(j)} } \frac{ 1}{ n p_{i} }  (\nabla  f_{i}(x_{t}^{(j)}) - \nabla f_{i}( x_{t-1}^{(j)} ) ) + \mathcal{V}_{t-1}^{(j)}$. We then update the variable $x_{t+1}^{(j)}$ based on stochastic variance reduced gradient $\mathcal{V}_{t}^{(j)}$ and the proximal mapping of $r(x)$.

	%Then we analyze the theoretical properties of Algorithm \ref{alg:SARAH-HT} for non-smooth non-convex regularized problems. 

	%In this section, we first propose the ProxSARAH with arbitrary sampling (AS)  (ProxSARAH-AS) in Algorithm \ref{alg:SARAH-HT}, which generalizes the original ProxSARAH \cite{xu2019stochastic}.  As shown in Algorithm \ref{alg:SARAH-HT}, there are inner loops in each outer loop. In an outer loop, a full gradient is computed in Line 3 for the use of reducing the variance of the stochastic gradients. In each inner loop, stochastic variance reduced gradient is calculated using a mini-batch $S_t^{(j)}$ that is drawn from $[n]$ according to $\mathbf{P}$. Then we analyze the theoretical properties of Algorithm \ref{alg:SARAH-HT} for non-smooth non-convex regularized problems. The original ProxSARAH is a special case of our formulation with uniform sampling, where the diagonal elements of $\mathbf{P}$ are $\frac{b}{n}$ and are zeros for other elements. Our new analysis actually helps show a better convergence speed for the original ProxSARAH method with a tighter bound.
	%{\bf what do you mean in the next sentence? are we the same as the existing result or our analysis is the same one as the existing way to do analysis?} Our theoretical analysis can cover existing non-asymptotic convergence analysis, but our bound is tighter.
	%, showing a better convergence compared with the known results in  \cite{xu2019stochastic}.
	
	\subsection{Unified analysis of ProxSARAH-AS}
	In this subsection, we provide the general convergence and computational complexity analysis for the ProxSARAH with uniform sampling and independent sampling respectively.
	
	%We present a general 
	%non-asymptotic 
	%convergence result for the ProxSARAH-AS, which can cover any proper sampling. 
	%Let $\mathcal{J}$ be the total number of epochs, $\Delta = F( \tilde{x}^1 ) - F( x^*)$, and define a constant $Q=\sum_{i=1}^{n}  \frac{v_iL_i^2}{p_i n^2}$. We then characterize our convergence result in Theorem \ref{th:sarah}.
	
	Similar to ProxSGD, the update of $x_{t+1}^{(j)}$ in ProxSARAH is: 
	\begin{align}
	x_{t+1}^{(j)} &\in \operatorname*{argmin}_{x \in \mathbb{R}^d} \bigg\{ r( x ) + \frac{1}{2\eta}\|x - ( x_t^{(j)} - \eta \mathcal{V}_t^{(j)})\|^2 \bigg\}  \label{eq:5}\\
	&= \operatorname*{argmin}_{x \in \mathbb{R}^d} \bigg\{ r( x ) + \langle \mathcal{V}_t^{(j)}, x - x_t^{(j)}\rangle +  \frac{1}{2\eta}\|x - x_t^{(j)} \|^2 \bigg\}, \nonumber
	\end{align}
	then by the definition of $\arg\min$, we have
	\begin{align*}
	0 \in \hat{\partial}r( x_{t+1}^{(j)} ) + \mathcal{V}_t^{(j)}  + \frac{1}{\eta}( x_{t+1}^{(j)} - x_t^{(j)}).
	\end{align*}
	Hence,
	$
	- \mathcal{V}_t^{(j)}  - \frac{1}{\eta}( x_{t+1}^{(j)}  - x_t^{(j)} )\in \hat{\partial}r( x_{t+1}^{(j)} ), 
	$
	implying 
	\begin{align}\label{eq:6}
	\nabla f( x_{t+1}^{(j)}) - \mathcal{V}_t^{(j)}  - \frac{1}{\eta}( x_{t+1}^{(j)}  - x_t^{(j)})
	&\in \nabla f( x_{t+1}^{(j)} )+ \hat{\partial}r( x_{t+1}^{(j)}) \nonumber\\
	&=  \hat{\partial}F( x_{t+1}^{(j)}  ).
	\end{align}
	
	\begin{algorithm}[H]
		\caption {ProxSARAH-AS }
		\label{alg:SARAH-HT} 
		\begin{algorithmic}[1]
			\Require Number of outer loops $\mathcal{J}$, inner loop $m$,  initial state $\tilde{x}^1$, stepsize $\eta$, probability matrix $\mathbf{P}$
			\For{$j=1,2,\dots,\mathcal{J}$}
			\State $x_0^{(j)} = \tilde{x}^{(j)}$
			\State $\mathcal{V}_{0}^{(j)}= \frac{1}{n} \sum_{i=1}^{n}\nabla f_{i}(x_0^{(j)})$ 
			\State $x_1^{(j)} = x_0^{(j)}$
			\For{$t=1,2,\dots, m$} 
			\State Draw a random subset $S^{(j)}_t  \subset \{1,...,n\}$ \State according to $\mathbf{P}$
			\State {$\mathcal{V}_{t}^{(j)} = \sum_{i \in S_t^{(j)} } \frac{ 1}{ n p_{i} }  (\nabla  f_{i}(x_{t}^{(j)}) - \nabla f_{i}( x_{t-1}^{(j)} ) )$ 
				\State $+ \mathcal{V}_{t-1}^{(j)}$ }
			\State $x_{t+1}^{(j)} = prox_{\eta r}[ x_{t}^{(j)} - \eta \mathcal{V}_{t}^{(j)} ]$
			\EndFor \label{alg:inner_e}
			\State set $\tilde{x}^{j+1} = x_{m+1}^{(j)}  $
			\EndFor\\
			\Return $x_{R}$, where $x_{R}$
			is uniformly sampled from $\{x^{(1)}_1, \ldots, x^{(\mathcal{J})}_m\}$
		\end{algorithmic}
	\end{algorithm}

	Before diving into the proof for the main theorem of convergence, we first give the following three lemmas as preparation.  Detailed proof can be found in supplemental material.
	%\textsuperscript{\ref{footnote}}.
	\begin{lemma}\label{lemma:Q}
		Suppose that Assumption \ref{assumption4f} holds and considering updating formula in ProxSARAH-AS: $\mathcal{V}_{t}^{(j)} = \sum_{i \in S_t^{(j)} } \frac{ 1}{ n p_{i} }  (\nabla  f_{i}(x_{t}^{(j)}) - \nabla f_{i}( x_{t-1}^{(j)} ) ) + \mathcal{V}_{t-1}^{(j)}$, then $\forall 1 \leq t\leq m$,  we have that  for any $j\geq 1$
		\begin{align*} 
		E[\|\mathcal{V}_t^{(j)} -  \nabla f( x_{t}^{(j)}) \|^2] \leq Q \sum_{k=1}^{t} E[\|x_{k}^{(j)}-x_{k-1}^{(j)}\|^2],
		\end{align*}
		where $Q=\sum_{i=1}^{n} \frac{v_iL_i^2}{p_i n^2}$. 
	\end{lemma}

	\begin{lemma} \label{lemma:F} Suppose that Assumption \ref{assumption4f} holds, we have that  for any  $\forall 1 \leq t\leq m$, $j\geq 1$
		\begin{align*} 
		\langle \mathcal{V}_t^{(j)} & - \nabla f( x_t^{(j)}  ), x_{t+1}^{(j)}  - x_t^{(j)} \rangle +  \frac{1}{2}(\frac{1}{\eta} -\tilde{L} )\|x_{t+1}^{(j)}  - x_t^{(j)}  \|^2 \\&
		\leq F( x_t^{(j)}  ) - F( x_{t+1}^{(j)} ).
		\end{align*}
	\end{lemma}

	\begin{lemma} \label{lemma:x}
		Suppose that Assumption \ref{assumption4f} holds, we have that  for any  $j\geq 1$
		\begin{align*}
		\sum_{t=1}^m  &E[\|x_{t+1}^{(j)}-x_{t}^{(j)}\|^2] 
		\\&\leq  \frac{1}{ \frac{1}{2}(\frac{1}{\eta} -2\tilde{L} ) - \frac{mQ}{2\tilde{L}}} E[ F( x_0^{(j)}  )  - F( x_{m+1}^{(j)} ) ], 
		\end{align*}
		where $\frac{1}{2}(\frac{1}{\eta} -2\tilde{L} ) - \frac{mQ}{2\tilde{L}} > 0$.
	\end{lemma}

	Using the above lemmas, we can state and prove our core convergence result for ProxSARAH-AS in Theorem \ref{th:sarah}, where we let $\mathcal{J}$ be the total number of epochs, $\Delta = F( \tilde{x}^1 ) - F( x^*)$, where $ \tilde{x}^1$ is the initial state and $x^*$ is the optimal of Problem (\ref{finite_sum_regularized}) and define a constant $Q=\sum_{i=1}^{n}  \frac{v_iL_i^2}{p_i n^2}$. 
	%We then characterize our convergence result in Theorem \ref{th:sarah}.

	\begin{theorem}\label{th:sarah}(Convergence guarantee for ProxSARAH-AS)
		Given Problem (\ref{finite_sum_regularized}), under Assumption \ref{assumption4f}, $\eta = \frac{1}{ 4\tilde{L}  + \frac{2mQ}{\hat{L}}}$ ,  the ProxSARAH-AS (Algorithm \ref{alg:SARAH-HT}) satisfies
		\begin{align*}
		\frac{1}{m\mathcal{J}}&\sum_{j=1}^{\mathcal{J}}\sum_{t=1}^m E[dist( 0, \hat{\partial}F( x_{t+1}^{(j)}))^2]  \leq  
		\frac{1}{m\mathcal{J}}(24 \tilde{L} 
		+\frac{4mQ}{\tilde{L}} )\Delta. 
		\end{align*}
	\end{theorem}

	\textit{Proof Sketch.}
	Let's try to bound $dist( 0, \hat{\partial}F( x_{t+1}^{(j)}  ))$. By Eq. (\ref{eq:6}),
	\begin{align}
	dist( 0, & \hat{\partial}F( x_{t+1}^{(j)}  ))^2
	=  \| \nabla f( x_{t+1}^{(j)}) - \mathcal{V}_t^{(j)}  - \frac{1}{\eta}( x_{t+1}^{(j)}  - x_t^{(j)} )\|^2 \nonumber \\
	&=  \| \nabla f( x_{t+1}^{(j)}) - \mathcal{V}_t^{(j)}\|^2  +   \frac{1}{\eta^2}\| x_{t+1}^{(j)}  - x_t^{(j)} \|^2 \nonumber\\&-\frac{2}{\eta}  \langle \nabla f( x_{t+1}^{(j)} ) - \mathcal{V}_t^{(j)}, x_{t+1}^{(j)}  - x_t^{(j)}  \rangle.  \label{eq:60}
	\end{align}
	Then by reorganizing inequality in Lemma \ref{lemma:F} , we obtain:
	\begin{align*}
	-\langle & \nabla f( x_{t+1}^{(j)} ) - \mathcal{V}_t^{(j)},  x_{t+1}^{(j)}  - x_t^{(j)} \rangle 
	\\&\leq F( x_t^{(j)}  ) - F( x_{t+1}^{(j)})-  \frac{1}{2}(\frac{1}{\eta} -\tilde{L} )\|x_{t+1}^{(j)}  - x_t^{(j)}  \|^2 \\ 
	&-\langle \nabla f( x_{t+1}^{(j)} )  - \nabla f( x_t^{(j)}  ), x_{t+1}^{(j)}  - x_t^{(j)} \rangle.
	\end{align*}
	Putting the above result in Eq. (\ref{eq:60}), and further applying $\tilde{L}$-smoothness of $f(x)$ and Young's inequality, we get,
	\begin{align*}
	&dist( 0, \hat{\partial}F( x_{t+1}^{(j)}  ))^2 \\
	%&\leq \| \nabla f( x_{t+1}^{(j)}) - \mathcal{V}_t^{(j)}\|^2  +  \frac{\tilde{L}}{\eta}\| x_{t+1}^{(j)}  - x_t^{(j)} \|^2 
	%+\frac{2}{\eta} (F( x_t^{(j)}) \\
	%&- F(x_{t+1}^{(j)})) - \frac{2}{\eta}\langle \nabla f( x_{t+1}^{(j)} )  - \nabla f( x_t^{(j)}  ), x_{t+1}^{(j)}  - x_t^{(j)} \rangle  \\
	%&\leq \| \nabla f( x_{t+1}^{(j)}) - \mathcal{V}_t^{(j)}\|^2  +  \frac{\tilde{L}}{\eta}\| x_{t+1}^{(j)}  - x_t^{(j)} \|^2 
	%+\frac{2}{\eta} (F( x_t^{(j)}) \\& -F(x_{t+1}^{(j)}))
	%+\frac{2}{\eta} \|\nabla f( x_{t+1}^{(j)} )  - \nabla f( x_t^{(j)}  )\| \| x_{t+1}^{(j)}  - x_t^{(j)} \| \\
	%&\leq \| \nabla f( x_{t+1}^{(j)}) - \mathcal{V}_t^{(j)}\|^2  +  \frac{\tilde{L}}{\eta}\| x_{t+1}^{(j)}  - x_t^{(j)} \|^2 \\& +\frac{2}{\eta} ( F( x_t^{(j)}  ) - F( x_{t+1}^{(j)} ) )+\frac{2\tilde{L}}{\eta}\| x_{t+1}^{(j)}  - x_t^{(j)} \|^2  \\
	%&=\|\nabla f( x_{t+1}^{(j)}) - \mathcal{V}_t^{(j)}\|^2  +  \frac{3\tilde{L}}{\eta}\| x_{t+1}^{(j)}  - x_t^{(j)} \|^2\\& + \frac{2}{\eta} ( F(x_t^{(j)}) - F( x_{t+1}^{(j)})) \\ 
	%&\leq 2\| \nabla f( x_{t}^{(j)}) - \mathcal{V}_t^{(j)}\|^2 + 2\| \nabla f( x_{t+1}^{(j)}) - \nabla f( x_{t}^{(j)})\|^2 \\& +  \frac{3\tilde{L}}{\eta}\| x_{t+1}^{(j)}  - x_t^{(j)} \|^2+\frac{2}{\eta} ( F( x_t^{(j)}  ) - F( x_{t+1}^{(j)}) ) \\ 
	&\leq 2\|\mathcal{V}_t^{(j)} -  \nabla f( x_{t}^{(j)}) \|^2  +  ( 2 \tilde{L}^2 +\frac{3\tilde{L}}{\eta})\| x_{t+1}^{(j)}  - x_t^{(j)} \|^2 \\&+\frac{2}{\eta} ( F( x_t^{(j)}  ) - F( x_{t+1}^{(j)})), 
	\end{align*}
	%where the third inequality and the last inequality are due to $\|\nabla f(x_{t+1}^{(j)})-\nabla f(x_t^{(j)})\|\leq \tilde{L}\|x_{t+1}^{(j)}-x_t^{(j)}\|$, and the penultimate inequality is due to Young's inequality $\|a+b\|^2 \leq 2\|a\|^2 +2\|b\|^2$.
	
	By summing over $t = 1, ..., m$, using the result in Lemma \ref{lemma:Q}, and taking the expectation, 
	\begin{align*}
	\sum_{t=1}^m E[&dist( 0, \hat{\partial}F( x_{t+1}^{(j)}))^2]\\  
	&\leq  (2Qm + 2\tilde{L}^2+\frac{3\tilde{L}}{\eta})
	\sum_{t=1}^m E[\| x_{t+1}^{(j)}  - x_t^{(j)} \|^2] \\&+\frac{2}{\eta}  E[F( x_0^{(j)}  ) - F( x_{m+1}^{(j)})].
	\end{align*}
	
	Next, plugging in  Lemma \ref{lemma:x} for the term $\sum_{t=1}^m E[\| x_{t+1}^{(j)}  - x_t^{(j)} \|^2] $,
	\begin{align*}
	&\sum_{t=1}^m E[dist( 0, \hat{\partial}F( x_{t+1}^{(j)}))^2]\\   
	& \leq   (\frac{2Qm + 2\tilde{L}^2+\frac{3\tilde{L}}{\eta}}{ \frac{1}{2}(\frac{1}{\eta} -2\tilde{L} ) - \frac{mQ}{2\tilde{L}}} 
	+\frac{2}{\eta} ) E[F( x_0^{(j)}  ) - F( x_{m+1}^{(j)})]. 
	\end{align*}
	Let $\frac{1}{2}\frac{1}{\eta} = 2\tilde{L}  + \frac{mQ}{\tilde{L}}$, then $\frac{1}{2}\frac{1}{\eta} > \tilde{L}  + \frac{mQ}{\tilde{L}}$, we get the final result.
	\qed
	
	We further analyze the computational complexity of the ProxSARAH-AS and obtain its computational complexity in terms of the IFO calls. Note that this part of our analysis does not need to go down into a specific sampling scheme, since we have a unifying form of $Q$ for different sampling schemes.
	\begin{theorem} \label{th:sarah_Complexity}{ (Complexity for ProxSARAH-AS)}
		In order to achieve an $\epsilon$-accuracy solution, i.e., $E[dist( 0, \hat{\partial}F( x_R))] \leq \epsilon$, the number of epochs required is $\mathcal{J} = \frac{1}{m\epsilon^2}(24 \tilde{L} 
		+\frac{4mQ}{\tilde{L}} )\Delta $, where $\Delta = F( \tilde{x}^1 ) - F( x^*)$. 
		The computational complexity in terms of the number of IFO calls is 
		$ \frac{ n + mb }{m\epsilon^2}(24 \tilde{L} 
		+\frac{4mQ}{\tilde{L}} ) \Delta.$
	\end{theorem}
	
	\subsection{ ProxSARAH with uniform sampling}
	For uniform sampling, we have  $p_i = \frac{b}{n}$ and $v_{i}=\frac{n-b}{n-1}$. With Theorem \ref{th:sarah_Complexity}, we are able to get the following corollary. 
	\begin{corollary} [Complexity with uniform sampling] \label{coro:sarah_unif}
		In order to have $E[dist( 0, \hat{\partial}F( x_R))] \leq \epsilon$, the number of epochs
		$\mathcal{J} = \frac{1}{m\epsilon^2}(24 \tilde{L} 
		+\frac{4m}{\tilde{L}} \frac{1}{b} \frac{1}{n} \frac{n-b}{n-1}(\sum_{i=1}^n L_i^2) )  \Delta$, where $\Delta = F( \tilde{x}^1 ) - F( x^*)$. 
		%The computational complexity (number of IFO calls) is 
		%\begin{align}
		%    & \frac{(n+mb)}{m\epsilon^2}(24 \tilde{L} 
		%     +\frac{4m}{\tilde{L}} \frac{1}{b} \frac{1}{n} \frac{n-b}{n-1}(\sum_{i=1}^n L_i^2) )  E[F( \tilde{x}^1  ) - F( x^*)].
		%\end{align}
		If we further assume $b=m=\sqrt{n}$, the number of IFO calls  is upper bounded by
		$\frac{\sqrt{n}}{\epsilon^2}(24 \tilde{L} 
		+\frac{4}{\tilde{L}}  \frac{1}{n}(\sum_{i=1}^n L_i^2) )  \Delta,$
		so the computational complexity is $O(\frac{1}{\epsilon^2})$.
	\end{corollary} 
	\begin{remark}
		If all $L_i$'s are the same and equal to $L$, the computational complexity is
		$\frac{28L\sqrt{n}}{\epsilon^2} \Delta.$
		Comparing with the results in \cite{xu2019stochastic}, $
		\frac{( \frac{4}{c} + 8c -2)/(1-3c)L\sqrt{n}}{\epsilon^2}  \Delta$,
		where $0< c < \frac{1}{3}$. Because $ (\frac{4}{c} + 8c -2)/(1-3c) \geq 46.67$, our bound is the tightest one so far.
	\end{remark}

	\subsection{ProxSARAH with independent sampling}
	For independent sampling case, we have $v_{i}=1-p_{i}$.
	To minimize $E[dist( 0, \hat{\partial}F( x_T))^2] $, again we need to optimize the following problem for the best $\mathbf{P^t}$:
	%\begin{align} \label{optim_sarah_2}
	%    \min _{\mathbf{p}, p_{i} \in[0,1], \sum_{i=1}^{n} p_{i}= b } \frac{1}{n^2} \sum_{i=1}^n \frac{ v_i }{ p_{i} } L_i^2
	%\end{align}
	%It is further equivalent to the following optimization problem:
	\begin{align} \label{optim_sarah_3}
	\min _{\mathbf{p}=\{ p_{i} \in[0,1]| \sum_{i=1}^{n} p_{i}= b \}} \frac{1}{n^2} \sum_{i=1}^n \frac{ 1 }{ p_{i} } L_i^2
	\end{align}

	Based on the KKT condition, the solution $\mathbf{P}^t$ to the above optimization problem is: 
	\begin{align}
	p_{i} :=\left\{\begin{array}{ll}{(b+k-n) \frac{L_{t}}{\sum_{j=1}^{k} L_{j},}} & {\text { if } i \leq k} \\ {1,} & {\text { if } i>k}\end{array}\right.
	\end{align}
	where $k$ is the largest integer satisfying $0< b+ k-n \leq \frac{\sum_{j=1}^{k} L_{j} }{L_{t}}$. If $L_i$'s for $i \in [1, n ]$ significantly differ one another so that %$L_1 << L_2 << ... << L_{n}$ and 
	$ 1<\frac{\sum_{j=1}^{k} L_{j} }{L_{k}} < 2$ for $k \in [n ]$, then $k = n -b + 1$ and
	\begin{align} \label{p_i}
	p_{i}=\left\{\begin{array}{ll}{ \frac{L_{i}}{\sum_{j=1}^{k} L_{j}},} & {\text { if } i \leq n-b+1} \\ {1,} & {\text { if } i > n-b+1.}\end{array}\right.
	\end{align}
	If $L_i$'s are similar to each other, so $bL_n \leq \sum_{j=1}^{n} L_{j}$, then $k=n$ and $p_{i} =\frac{b L_{i}}{\sum_{j=1}^{n} L_{j}} $ for $ i \in [n]$.
	
	We also obtain the following specific corollary for independent sampling:
	
	\begin{corollary} [Complexity with independent sampling] \label{coro:sarah_ind}
		In order to have $E[dist( 0, \hat{\partial}F( x_R))] \leq \epsilon$, the number of epochs is 
		\begin{align}
		\mathcal{J} &= \frac{1}{m\epsilon^2}(24 \tilde{L} 
		+\frac{4m}{\tilde{L}}  \frac{C_1}{n^2}   \frac{1}{b+k-n}(\sum_{i=1}^{k}  L_{i})^2   \Delta).
		\end{align} 
		%The computational complexity (number of IFO calls) is 
		%\begin{align}
		%    & \frac{(n+mb)}{m\epsilon^2}(24 \tilde{L} 
		%     +\frac{4m}{\tilde{L}}  \frac{C_1}{n^2} ( \frac{1}{b+k-n}(\sum_{i=1}^{k}  L_{i})^2 \nonumber\\
		%     &- \sum_{i=1}^{k} L_i^2 ) ) E[F( \tilde{x}^1  ) - F( x^*)].
		%\end{align}
		%where $Q=\sum_{i=1}^{n} \frac{v_iL_i^2}{p_i n^2}$ .
		If we further assume  $bL_n \leq \sum_{j=1}^{n} L_{j}$, and $b=m=\sqrt{n}$, the number of IFO calls is bounded by 
		$ \frac{\sqrt{n}}{\epsilon^2}(24 \tilde{L} 
		+\frac{4}{\tilde{L}}  \frac{C_1}{n^2} (\sum_{i=1}^{n}  L_{i})^2 \Delta)$,
		so the computational complexity is $O(\frac{1}{\epsilon^2})$.
	\end{corollary} 
	\begin{remark}
		Results in Corollaries \ref{coro:sarah_unif} and \ref{coro:sarah_ind}
		imply that the independent sampling scheme improves the computational complexity, because
		\begin{align*}
		\frac{(\sum_{i=1}^n L_i^2)}{n}/\frac{(\sum_{i=1}^{n} L_{i})^2 }{n^2} = \frac{n \sum_{i=1}^n L_i^2}{(\sum_{i=1}^{n} L_{i})^2} \geq 1.
		\end{align*}
	\end{remark}
	
	\section{The ProxSPIDER with Arbitrary Sampling}
	
	In this section, we further propose a new method, ProxSPIDER-AS %, by integrating the AS scheme   with the original ProxSPIDER method \cite{xu2019stochastic}, which is known as another representative of variance reduced methods, 
	to speed up the convergence process of solving non-convex non-smooth regularized problems. We also provide the convergence and computational complexity results under our unified analytic approach.

	The details of ProxSPIDER-AS are given in Algorithm \ref{alg:SPIDER-HT}. The key difference between Algorithm \ref{alg:SARAH-HT} and \ref{alg:SPIDER-HT} is that ProxSPIDER-AS, unlike ProxSARAH-AS, avoids computation of the full gradient, which can be computationally prohibitive for massive datasets. Instead, it calculates a batch gradient over a mini-batch $S^{(j)}$ for variance reduction. Specifically, at the beginning of each outer loop iteration $j$, we estimate the gradient $\sum_{i \in S^{(j)} } \frac{ 1}{ n p'_{i} } \nabla f_{i}(x_0^{(j)})$ over a random subset $S^{(j)}$ with batch size $B$, which are
	sampled from $[n]$ based on an arbitrary distribution $\mathbf{P'}$. In the following inner loop iterations, we construct the stochastic gradient estimator $\mathcal{V}_{t}^{(j)}$ based on a subset data samples $S^{(j)}_t $ draw from $[n]$ according to a probability matrix $\mathbf{P}$, i.e., 
	$\mathcal{V}_{t}^{(j)} = \sum_{i \in S_t^{(j)} } \frac{ 1}{ n p_{i} }  (\nabla  f_{i}(x_{t}^{(j)}) - \nabla f_{i}( x_{t-1}^{(j)} ) ) + \mathcal{V}_{t-1}^{(j)}$. In order to handle the possible non-smoothness, we then perform a proximal gradient step to update the variable, i.e., $x_{t+1}^{(j)} = prox_{\eta r}[ x_{t}^{(j)} - \eta \mathcal{V}_{t}^{(j)} ]$.

	%{\color{red}The ProxSPIDER method can also incorporate the arbitrary sampling scheme to give the ProxSPIDER-AS in Algorithm \ref{alg:SPIDER-HT}. Different from the ProxSARAH-AS, the ProxSPIDER-AS calculates a batch gradient over a mini-batch $S^{(j)}$ for variance reduction in Line 5 rather than the full gradient in Algorithm \ref{alg:SARAH-HT} which can be computationally prohibitive for massive datasets. Thus, in Algorithm \ref{alg:SPIDER-HT}, two sampling schemes may be needed to sample $S^{(j)}$ (using the respective $p'$ and $v'$)  and $S^{(j)}_t$ (based on $p$ and $v$). }

	\begin{algorithm}[!htb]
		\caption {ProxSPIDER-AS}
		\label{alg:SPIDER-HT} 
		\begin{algorithmic}[1]
			\Require Number of outer loops $\mathcal{J}$, inner loop $m$,  initial state $\tilde{x}^1$, stepsize $\eta$, probability matrices $\mathbf{P}$, $\mathbf{P'}$
			\For{$j=1,2,\dots,\mathcal{J}$}
			\State $x_0^{(j)} = \tilde{x}^{(j)}$
			\State Draw a random subset $S^{(j)} \subset \{1,...,n\}$ with
			\State  size $B$, according to $\mathbf{P'}$
			\State $\mathcal{V}_{0}^{(j)}=  \sum_{i \in S^{(j)} } \frac{ 1}{ n p'_{i} } \nabla f_{i}(x_0^{(j)})$ 
			\State $x_1^{(j)} = x_0^{(j)}$
			\For{$t=1,2,\dots, m$} 
			\State Draw a random subset $S^{(j)}_t \subset \{1,...,n\}$
			\State with size $b$, according to $\mathbf{P}$
			\State { $\mathcal{V}_{t}^{(j)} = \sum_{i \in S_t^{(j)} } \frac{ 1}{ n p_{i} }  (\nabla  f_{i}(x_{t}^{(j)}) - \nabla f_{i}( x_{t-1}^{(j)} ) ) $
				\State $+ \mathcal{V}_{t-1}^{(j)}$ }
			\State $x_{t+1}^{(j)} = prox_{\eta r}[ x_{t}^{(j)} - \eta \mathcal{V}_{t}^{(j)} ]$
			\EndFor \label{alg:inner_e}
			\State set $\tilde{x}^{j+1} = x_{m+1}^{(j)}  $
			\EndFor\\
			\Return $x_{R}$, where $x_{R}$
			is uniformly sampled from $\{x^{(1)}_1, \ldots, x^{(\mathcal{J})}_m\}$
		\end{algorithmic}
	\end{algorithm}
	
	\subsection{Unified analysis of ProxSPIDER-AS}
	In this subsection, we will provide the general convergence analysis for the ProxSPIDER-AS, which can cover any proper sampling.
	Before showing the convergence result of the ProxSPIDER-AS, we first provide the following preparation lemmas to help the understanding of main theorem.  Detailed proof can be found in supplemental material.
	%\textsuperscript{\ref{footnote}}.
	
	%Let $\Delta = F( x_1 ) - F( x^*)$, $Q=\sum_{i=1}^{n} \frac{v_iL_i^2}{p_i n^2}$, and $Q'=\sum_{i=1}^{n} \frac{v_iG_i^2}{p_i n^2} $. We characterize the non-asymptotic convergence result of the ProxSPIDER-AS using $\Delta$, $Q$ and $Q'$ in the following theorem.
	
	%Similarly, we have the following preparation lemmas.
	
	\begin{lemma}\label{lemma:Q2}
		Suppose that Assumption \ref{assumption4f} holds and consider updating formula in ProxSPIDER-AS: $\mathcal{V}_{t}^{(j)} = \sum_{i \in S_t^{(j)} } \frac{ 1}{ n p_{i} }  (\nabla  f_{i}(x_{t}^{(j)}) - \nabla f_{i}( x_{t-1}^{(j)} ) ) + \mathcal{V}_{t-1}^{(j)}$,  $Q=\sum_{i=1}^{n} \frac{v_iL_i^2}{p_i n^2}$, then $\forall 1 \leq t\leq m$, we have  for any $j\geq 1$
		\begin{align*} 
		E[\|\mathcal{V}_t^{(j)} -  \nabla f( x_{t}^{(j)}) \|^2] &\leq Q \sum_{k=1}^{t} E[\|x_{k}^{(j)}-x_{k-1}^{(j)}\|^2]+ Q'.
		\end{align*}
		where $Q=\sum_{i=1}^{n} \frac{v_iL_i^2}{p_i n^2}$ and $Q' = \sum_{i=1}^n \frac{ v'_i G_i^2 }{ p'_{i} n^2 }  $. 
	\end{lemma}
	
	\begin{lemma} \label{lemma:x2}
		Suppose that Assumption \ref{assumption4f} holds, for  $j\geq 1$, we have $\sum_{t=1}^m  E[\|x_{t+1}^{(j)}-x_{t}^{(j)}\|^2] \leq 
		\frac{1}{ \frac{1}{2}(\frac{1}{\eta} -2\tilde{L} ) - \frac{mQ}{2\tilde{L}}} E[ F( x_0^{(j)}  )  - F( x_{m+1}^{(j)} ) ],$ where $\frac{1}{2}(\frac{1}{\eta} -2\tilde{L} ) - \frac{mQ}{2\tilde{L}} > 0$.
		\iffalse
		\begin{align*} 
		\sum_{t=1}^m  E[\|&x_{t+1}^{(j)}-x_{t}^{(j)}\|^2] \\&\leq 
		\frac{1}{ \frac{1}{2}(\frac{1}{\eta} -2\tilde{L} ) - \frac{mQ}{2\tilde{L}}} E[ F( x_0^{(j)}  )  - F( x_{m+1}^{(j)} ) ], 
		
		%- \frac{ mQ'}{ \frac{1}{2}(\frac{1}{\eta} -2\tilde{L} ) - \frac{mQ}{2\tilde{L}}},
		\end{align*}\fi
		%where $\frac{1}{2}(\frac{1}{\eta} -2\tilde{L} ) - \frac{mQ}{2\tilde{L}} > 0$.
	\end{lemma}
	
	With the above preparations, we are able to derive the following generic convergence result of the ProxSPIDER-AS for any proper sampling using $\Delta$, $Q$ and $Q'$, where $\Delta = F( \tilde{x}^1 ) - F( x^*)$, $Q=\sum_{i=1}^{n} \frac{v_iL_i^2}{p_i n^2}$, and $Q'=\sum_{i=1}^{n} \frac{v_iG_i^2}{p_i n^2}$.
	
	\begin{theorem}\label{th:spider} (Convergence guarantee for ProxSPIDER-AS)
		Given Problem (\ref{finite_sum_regularized}), under Assumption 1, 
		let $\eta = 1/( 4\tilde{L}  + 2mQ/\hat{L})$,  then for $\forall 1 \leq t\leq m$, $j\geq 1$, the ProxSPIDER-AS (Algorithm \ref{alg:SPIDER-HT}) achieves
		\begin{align*}
		\frac{1}{m\mathcal{J}}\sum_{j=1}^{\mathcal{J}}\sum_{t=1}^m & E[dist( 0, \hat{\partial}F( x_{t+1}^{(j)}))^2]  \\
		&\leq  \frac{1}{m\mathcal{J}}(24 \tilde{L} 
		+\frac{4mQ}{\tilde{L}} ) \Delta+ 2Q'. 
		\end{align*}
		
	\end{theorem}
	
	\textit{Proof Sketch.}
	Let's try to bound $dist( 0, \hat{\partial}F( x_{t+1}^{(j)}  ))$. Similar with the analysis of ProxSARAH-AS,
	\begin{align*}
	&dist( 0, \hat{\partial}F( x_{t+1}^{(j)}  ))^2 
	\leq 2\|\mathcal{V}_t^{(j)} -  \nabla f( x_{t}^{(j)}) \|^2 \\&  +  ( 2 \tilde{L}^2 +\frac{3\tilde{L}}{\eta})\| x_{t+1}^{(j)}  - x_t^{(j)} \|^2 +\frac{2}{\eta} ( F( x_t^{(j)}  ) - F( x_{t+1}^{(j)})), 
	\end{align*}
	
	By summing over $t = 1, ..., m$, using the result in Lemma \ref{lemma:Q2}, and taking the expectation, 
	\begin{align*}
	& \sum_{t=1}^m E[dist( 0, \hat{\partial}F( x_{t+1}^{(j)}))^2]   \\ 
	& \leq  (2Qm + 2\tilde{L}^2+\frac{3\tilde{L}}{\eta})
	\sum_{t=1}^m E[\| x_{t+1}^{(j)}  - x_t^{(j)} \|^2] \\&+\frac{2}{\eta}  E[F( x_1^{(j)}  ) - F( x_{m+1}^{(j)})] + 2mQ' 
	\end{align*}
	
	Next, plugging in  Lemma \ref{lemma:x2} for the term $\sum_{t=1}^m E[\| x_{t+1}^{(j)}  - x_t^{(j)} \|^2] $,
	\begin{align*}
	& \sum_{t=1}^m E[dist( 0, \hat{\partial}F( x_{t+1}^{(j)}))^2]  \\ 
	& \leq (\frac{2Qm + 2\tilde{L}^2+\frac{3\tilde{L}}{\eta}}{ \frac{1}{2}(\frac{1}{\eta} -2\tilde{L} ) - \frac{mQ}{2\tilde{L}}} 
	+\frac{2}{\eta} ) E[F( x_0^{(j)}  ) - F( x_{m+1}^{(j)})] +2mQ'.
	%\textcolor{red}{- missing item}. 
	\end{align*}
	
	Therefore:
	\begin{align*}
	&\frac{1}{m\mathcal{J}}\sum_{j=1}^{\mathcal{J}}\sum_{t=1}^m E[dist( 0, \hat{\partial}F( x_{t+1}^{(j)}))^2] \\& \leq
	\frac{1}{m\mathcal{J}}(\frac{2Qm + 2\tilde{L}^2+\frac{3\tilde{L}}{\eta}}{ \frac{1}{2}(\frac{1}{\eta} -2\tilde{L} ) - \frac{mQ}{2\tilde{L}}} 
	+\frac{2}{\eta} ) E[F( \tilde{x}^1  ) - F( x^*)] + 2Q'. 
	\end{align*}
	Let $\frac{1}{2}\frac{1}{\eta} = 2\tilde{L}  + \frac{mQ}{\tilde{L}}$,then $\frac{1}{2}\frac{1}{\eta} > \tilde{L}  + \frac{mQ}{\tilde{L}}$ and we get the final result.
	\qed
	
	For simplicity, we use the same sampling scheme, either the uniform sampling or independent sampling, for drawing both $S^{(j)}$ and $S^{(j)}_t$.
	
	\begin{figure*}[h]
		\centering
		\includegraphics[trim=80 20 80 0,clip,width=0.9\textwidth]{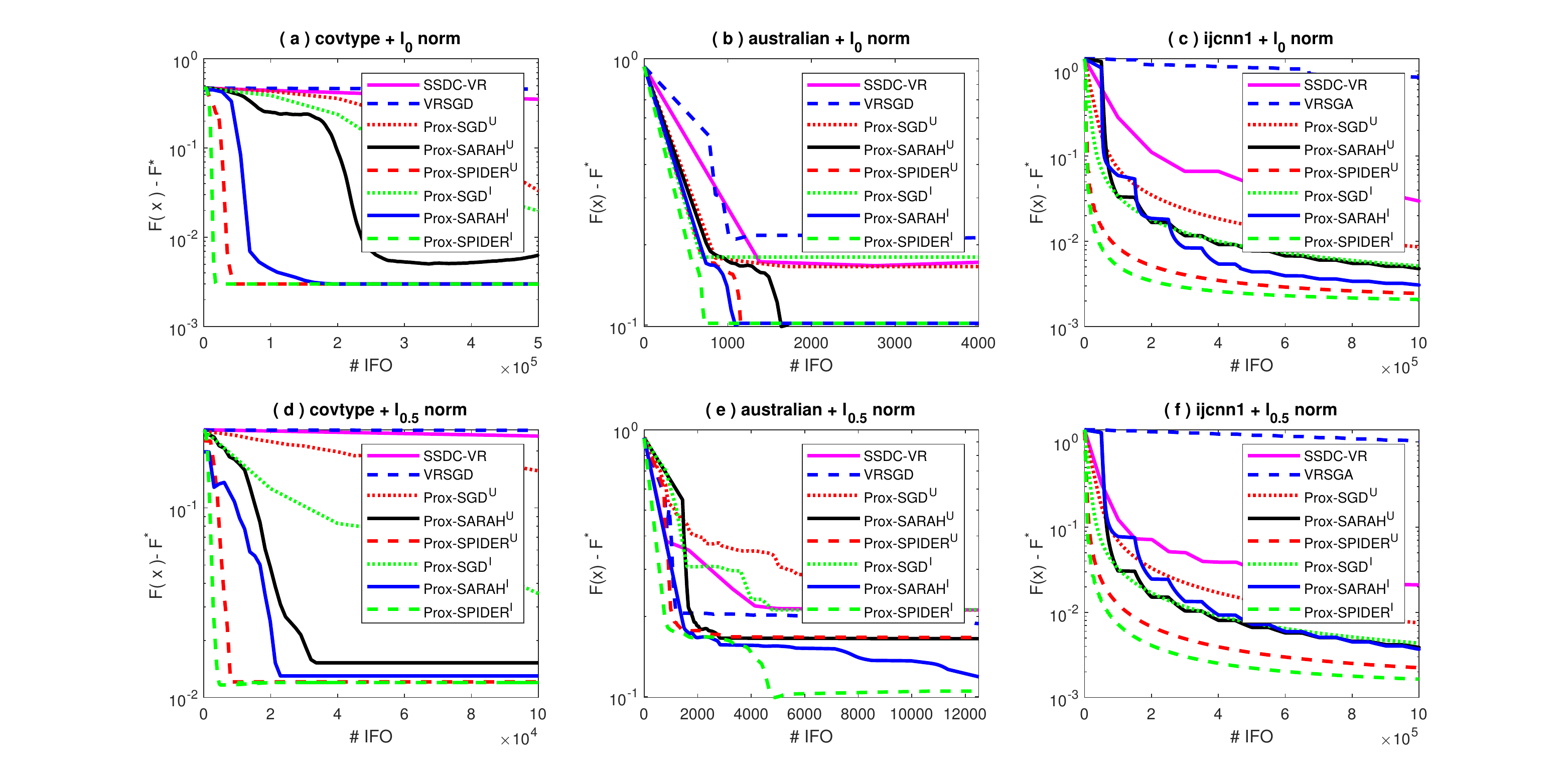}
		\caption{Comparisons of Prox-SGD$^U$, Prox-SARAH$^U$, Prox-SPIDER$^U$, Prox-SGD$^I$,  Prox-SARAH$^I$ and  Prox-SPIDER$^I$,  with SSDC-VR and VRSGD on $l_0$ and $l_{0.5}$ regularized problems. $F^*$ is a lower bound of function $F(x)$.
		}
		\label{fig}
		\vspace{-0.5cm}
	\end{figure*}
	
	\subsection{ProxSPIDER with uniform and independent sampling }
	Again, we obtain the following corollaries on computational complexity if a specific sampling scheme is used.
	
	\begin{corollary}[Complexity with uniform sampling]
		Given Problem (\ref{finite_sum_regularized}), under Assumption 1, if $B =\frac{2}{n\epsilon^2}(\sum_{i=1}^n G_i^2) $, $m = b = \sqrt{B}$, $\Delta = F( \tilde{x}^1 ) - F( x^*)$, then the number of IFO calls to achieve $E[dist( 0, \hat{\partial}F( x_R))] \leq \epsilon$ is bounded by 
		$
		\frac{2\sqrt{B}}{\epsilon^2}(24 \tilde{L} 
		+\frac{4}{\tilde{L}} \frac{1}{n}(\sum_{i=1}^n L_i^2) ) \Delta,%= O( \frac{1}{\epsilon^3} ) .
		$ and the computational complexity is $O(\frac{1}{\epsilon^3})$. 
	\end{corollary}
	
	%\begin{remark}
	%Similarly, we obtain a tighter bound than the one in \cite{xu2019stochastic}: $\frac{( \frac{4}{c} + 8c -2)/(1-3c)L\sqrt{n}}{\epsilon^2}  \Delta$,
	%where $0< c < \frac{1}{3}$. Because our results do not depend on $\frac{1}{c}$, where $0< c < \frac{1}{3}$, our bound is the tightest one so far.????????
	%\end{remark}
	
	For the independent sampling scheme, we further assume there is no significant difference in $G_i$'s and $L_i$'s (  i.e., $bG_n \leq \sum_{j=1}^{n} G_{j}$  and  $bL_n \leq \sum_{j=1}^{n} L_{j}$ ), we obtain the following result.
	
	\begin{corollary}[Complexity with independent sampling]
		Given Problem (\ref{finite_sum_regularized}), under Assumption 1, if $B =\frac{2}{n^2\epsilon^2}(\sum_{i=1}^n G_i)^2 $, $m = b = \sqrt{B}$, $\Delta = F( \tilde{x}^1 ) - F( x^*)$, then the number of IFO calls required to achieve $E[dist( 0, \hat{\partial}F( x_R))] \leq \epsilon$ is bounded by
		$
		\frac{2\sqrt{B}}{\epsilon^2}(24 \tilde{L} 
		+\frac{4}{\tilde{L}} \frac{1}{n^2 }(\sum_{i=1}^n L_i)^2 ) \Delta.% = O( \frac{1}{\epsilon^3} ).
		$ so the computational complexity is $O(\frac{1}{\epsilon^3})$.
	\end{corollary}

	\section{Experiments}
	
	We empirically compare the proposed algorithms against the state-of-the-art methods: SSDC-VR \cite{xu2018stochastic},  and VRSGD \cite{metel2019simple}. For clarity, we use the superscript U to denote the methods with uniform sampling (Prox-SGD$^U$, Prox-SARAH$^U$, Prox-SPIDER$^U$) and the superscript I  to denote those with independent sampling (Prox-SGD$^I$, Prox-SARAH$^I$, Prox-SPIDER$^I$). %{\color{red}Note that the algorithms in \cite{xu2019stochastic} are exactly the same as Prox-SGD$^U$, Prox-SARAH$^U$, and Prox-SPIDER$^U$, so we do not further differentiate them with ours.}
	Three benchmark datasets are used in our experiments:  covtype, australian and ijcnn1, all of which can be downloaded from the LibSVM website\footnote{http://www.csie.ntu.edu.tw/cjlin/libsvmtools/datasets/}.
	
	%We compare the proposed algorithms: stochastic proximal gradient with uniform sampling (SPG$^U$), ProxSARAH with uniform sampling (Prox-SARAH$^U$),  ProxSPIDER with uniform sampling (Prox-SPIDER$^U$),  stochastic proximal gradient with independent sampling (SPG$^I$),  ProxSARAH with independent sampling (Prox-SARAH$^I$) and ProxSPIDER with independent sampling (Prox-SPIDER$^I$),  with the state-of-the-art SSDC-VR and VRSGD , in our experiments to demonstrate the improved performance and advantages.
	Following the convention in the stochastic optimization literature, we use the number of the IFO calls to measure the computational complexity. 
	This can make the computational complexity independent of actual implementation of an algorithm. 
	For a comprehensive comparison, we also include the decrease of function values over the number of iterations into our comparison. 
	The parameter $\mu$ for SSDC-VR and VRSGD is chosen by grid search from $\{10, 1, 0.1, 0.01, 0.001\}$.
	The stepsize $\eta$ for each algorithm is set by a grid search from $\{10, 1, 10^{-1}, 10^{-2}, 10^{-3}, 10^{-4}\}$. All the algorithms are initialized with the same $x^{(0)}$ for the same dataset.
	
	In our experiments, we use different algorithms to solve the following problem for classification tasks:
	\begin{equation*} 
	\min_{x \in \mathbb{R}^d} \left\{ F( x ) := f(x) + \lambda r(x) \right\}.
	\end{equation*} We adopt the following smooth but non-convex regression function $f( x )$ to be the classification loss function:
	$
	f(x)=\frac{1}{n} \sum_{i=1}^{n}\left(1-y_{i} \sigma\left(a_{i}^{\top} x\right)\right)^{2}
	$
	where $\sigma( \cdot )$ is the sigmoid function. This function has been extensively used to test stochastic algorithms with different sampling techniques because $G_i$'s and $L_i$'s for this function can be  computed via $\|x\|_2$. These parameters may be estimated, for some more complex problems \cite{horvath2019nonconvex}.  The non-smooth and non-convex $l_p$ norm is used as the regularizer $r( x ) $ where $0 \leq p < 1 $, which are commonly used in sparse learning. More specifically,  $p=0$ and $p=0.5$ are tested due to their well-studied proximal operators.
	
	%As a reflection of our Remark \ref{remark:sgd_comp}, SPG$^I$ is almost always faster than SPG$^U$ in Figure 1, except results in Figure 1 (b).  The exception is because two algorithms converge to different local minimal. By comparing all algorithms, we can see that the proposed algorithm Prox-SPIDER$^I$ is the fastest across different tasks. As shown in Figure 1 (c), Prox-SPIDER$^I$ not only can drop loss function quickly at the beginning stage, but also can escape bad stationary points quickly.  In one word, all methods with independent sampling technique are faster than their corresponding ones with uniform sampling. These results are consistent with our theory.
	
	In our experiments, both SSDC-VR and VRSGD do not perform well for covtype data and VRSGD also reduce function values slower than other algorithms for ijcnn1 data. Compared with the SSDC-VR and VRSGD methods, the Prox-SARAH$^I$ and Prox-SPIDER$^I$ show obvious improvements over the counterparts with uniform sampling. We can also see that the proposed Prox-SPIDER$^I$ algorithm is the fastest among the algorithms across the different tasks. An interesting observation shown in Figure 1 (e) is that the Prox-SPIDER$^I$ can not only reduce the loss function quickly at the beginning stage, but also escape narrow stationary points in the later stage, which may benefit from the variation of the batch gradient in the outer loop. 
	%As a reflection of our theory, SPG$^I$ is always faster than SPG$^U$ in experiments.  
	%The exception may be because two algorithms converge to different local minimal.
	Based on the experimental results, it is safe to conclude that all methods using the independent sampling technique tend to be faster than their corresponding methods with uniform sampling. These empirical observations are consistent with our theoretical results.
	
	\section{Conclusion}
	
	%We have proposed stochastic proximal gradient method and its variants with arbitrary sampling techniques. We provide the first independent sampling variants of stochastic algorithms for empirical risk minimization with non-convex non-smooth regularizer, which achives the state-of-the-art computational complexity. In particular, we analyze non-convex and non-smooth regularized versions of stochastic proximal  gradient methods : Prox-SGD, Prox-SARAH and Prox-SPIDER. What's more, we also provide a general analysis of those proximal methods with arbitrary sampling for non-convex non-smooth regularized problems. Right now, all researches are focused on $\epsilon$-stationary points. Hence, further exploration for escaping saddle points or converging to local minimal are needed and will be our future research direction.
	
	To solve the sparse learning problems with  nonconvex nonsmooth regularization, we propose a series of stochastic proximal gradient methods, including ProxSGD-AS, ProxSARAH-AS, and ProxSPIDER-AS, that replace the original methods by a new data sampling scheme. The proposed methods draw mini-batches based on an arbitrary probability distribution when calculating the stochastic gradients. A unified analytic approach is developed to examine the convergence and computational complexity of these methods when an arbitrary sampling scheme is adopted.  This theoretical framework helps us compare the different sampling schemes, and we show that these proximal methods tend to perform better when independent sampling is used rather than uniform sampling. Furthermore, even for the uniform sampling, our new analysis derives a tighter bound on convergence speed than the best one available so far. Empirical studies confirm our theoretical observations. As a future direction, since the current research is focused on $\epsilon$-stationary points, further exploration for escaping saddle points or converging to local minimal might be of great interest for non-convex problems.
	\section*{Acknowledgment}
	We thank the reviewers for their insightful comments. The work of G. Liang, Q. Tong, and J. Bi was funded by NSF grants CCF-1514357 and IIS-1718738. J. Bi was also supported by NIH grants 5K02DA043063-03 and 1R01MH119678-01. The work of J. Ding and M. Pan was supported in part by the U.S. National Science Foundation under grants US CNS-1646607, CNS-1801925, and CNS-2029569.
	\bibliographystyle{IEEEtran}
	\bibliography{SPG_sampling}
	\newpage
	\onecolumn
	\section*{Appendix}
	\setcounter{section}{0}
	\section{Minibatch ProxSGD with Arbitrary Sampling}
	We have stated the update of ProxSGD in main paper, here we restate the main results we may use in prove procedure:
	\begin{align}
	x_{t+1} &\in \operatorname*{argmin}_{x \in \mathbb{R}^d} \bigg\{ r( x ) + \frac{1}{2\eta}\|x - ( x_t - \eta g_t)\|^2 \bigg\} \nonumber \\
	&= \operatorname*{argmin}_{x \in \mathbb{R}^d} \bigg\{ r( x ) + \langle g_t, x - x_t\rangle +  \frac{1}{2\eta}\|x - x_t \|^2 \bigg\}, \label{ineq:argmin1}
	\end{align}
	\begin{align}\label{eq:1}
	\nabla f( x_{t+1}) - g_t  - \frac{1}{\eta}( x_{t+1} - x_t)  \in \nabla f( x_{t+1} )+ \hat{\partial}r( x_{t+1} ) =  \hat{\partial}F(x_{t+1}),
	\end{align}
	
	\noindent\textbf{Lemma 4.1.} [distance bound]
	\begin{align*} 
	\left\| {x}_{t+1}- {x}_{t}\right\|^{2}   
	\leq \frac{2\eta}{1-2\tilde{L}\eta}(F\left( {x}_{t}\right)-F\left( {x}_{t+1}\right)) 
	+\frac{\eta}{\tilde{L}-2 \tilde{L}^2\eta}\left\|{g}_{t}-\nabla f\left( {x}_{t}\right)\right\|^{2}.
	\end{align*}
	
	\begin{proof}
		Considering equality (\ref{ineq:argmin1}), we  obtain:
		\begin{align}\label{ineq:r_function}
		r\left( {x}_{t+1}\right)+\left\langle{g}_{t},  {x}_{t+1}- {x}_{t}\right\rangle+\frac{1}{2 \eta}\left\| {x}_{t+1}- {x}_{t}\right\|^{2} \leq r\left( {x}_{t}\right)
		\end{align}
		Since $f( \cdot)$ is $\tilde{L}$-smoothness, we get:
		\begin{align} \label{ineq:L_smoothness}
		f\left( {x}_{t+1}\right) \leq f\left( {x}_{t}\right)+\left\langle\nabla f\left( {x}_{t}\right),  {x}_{t+1}- {x}_{t}\right\rangle+\frac{\tilde{L}}{2}\left\| {x}_{t+1}- {x}_{t}\right\|^{2}
		\end{align}
		Combining inequalities (\ref{ineq:r_function}) and (\ref{ineq:L_smoothness}), we finally get: 
		\begin{align*}
		\left\langle{g}_{t}-\nabla f\left( {x}_{t}\right),  {x}_{t+1}- {x}_{t}\right\rangle&+\frac{1}{2}(1 / \eta-\tilde{L})\left\| {x}_{t+1}- {x}_{t}\right\|^{2} \leq F\left( {x}_{t}\right)-F\left( {x}_{t+1}\right)
		\end{align*}
		Then, 
		\begin{align*} 
		\frac{1}{2}(1/\eta-\tilde{L})\left\| {x}_{t+1}- {x}_{t}\right\|^{2}  
		& \leq F\left( {x}_{t}\right)-F\left( {x}_{t+1}\right)-\left\langle {g}_{t}-\nabla f\left( {x}_{t}\right),  {x}_{t+1}- {x}_{t}\right\rangle \nonumber\\ 
		& \leq F\left( {x}_{t}\right)-F\left( {x}_{t+1}\right)+\frac{1}{2 \tilde{L}}\left\|{g}_{t}-\nabla f\left( {x}_{t}\right)\right\|^{2} +\frac{\tilde{L}}{2}\left\| {x}_{t+1}- {x}_{t}\right\|^{2}. 
		\end{align*}
		The last inequality holds due to $-\langle a, b \rangle \leq  \frac{1}{2c}\|a\|^2 + \frac{c}{2}\|b\|^2$.
		Therefore we achieve the bound of distance of parameter $x$,
		\begin{align*} 
		\left\| {x}_{t+1}- {x}_{t}\right\|^{2}   
		\leq \frac{2\eta}{1-2\tilde{L}\eta}(F\left( {x}_{t}\right)-F\left( {x}_{t+1}\right)) 
		+\frac{\eta}{\tilde{L}-2 \tilde{L}^2\eta}\left\|{g}_{t}-\nabla f\left( {x}_{t}\right)\right\|^{2}. 
		\end{align*}
	\end{proof}

	Then we give the details of proof of our main results in Theorem 4.2.
	
	\noindent\textbf{Theorem 4.2.}
	Considering Problem (1) under Assumption 1, if $ 0 < \eta < \frac{1}{2 \tilde{L}}$, then for all $T \geq 1$, ProxSGD-AS (Algorithm 1) will have,
	\begin{align*}
	E[dist( 0, \hat{\partial}F( x_T))^2] 
	\leq \frac{C_1}{T} \sum_{t=1}^T E[\|\nabla f(x_t) -g_t\|^2]  + \frac{C_2}{T}\Delta,
	\end{align*}
	where $C_1= \frac{1+ 4\tilde{L}\eta - 2\tilde{L}^2\eta^2}{\tilde{L}\eta-2\tilde{L}^2\eta^2}$,
	$C_2 = \frac{2+4\tilde{L}\eta + 4\tilde{L}^2\eta^2}{\eta - 2\tilde{L}\eta^2}$, $\Delta = F( x_1 ) - F( x^*)$. 
	
	\begin{proof}
		Let $\xi_i = \nabla f_i (x_t)$, then from Lemma D.1,
		\begin{align}\label{eq:3}
		E[g_t] = E\bigg[\sum_{i\in S_t} \frac{\xi_i}{np_i}\bigg] = \tilde{\xi} = \nabla f(x_t),
		\end{align}
		and
		\begin{align} \label{eq:4}
		E[\|g_t - \nabla f(x_t)\|^2] \leq \frac{1}{n^2} \sum_{i=1}^n \frac{v_i}{p_i}\|\nabla f_i(x_t)\|^2.   
		\end{align}
		From the above analysis, in order to measure the sub-gradient of regularized non-smooth non-convex problem, we
		now considering $dist( 0, \hat{\partial}F( x_{t+1} ))$. By Eq (\ref{eq:1}), 
		\begin{align*}
		dist( 0, \hat{\partial}F( x_{t+1}))^2 
		&\leq  \| \nabla f( x_{t+1}) - g_t -\frac{1}{\eta}( x_{t+1} - x_t)\|^2 \nonumber\\
		&=  \| \nabla f(x_{t+1})-g_t\|^2+\frac{1}{\eta^2}\| x_{t+1} - x_t\|^2
		-\frac{2}{\eta}\langle\nabla f(x_{t+1})-g_t, x_{t+1}-x_t \rangle
		\end{align*}
		Taking the expectation on both sides,
		\begin{align*}
		&E[dist( 0, \hat{\partial}F( x_{t+1}))^2] \\
		&\leq E[\| \nabla f(x_{t+1})-g_t\|^2] + \frac{1}{\eta^2}E[\| x_{t+1} - x_t\|^2] 
		-\frac{2}{\eta}E[\langle\nabla f(x_{t+1})-g_t, x_{t+1}-x_t \rangle]\\
		& = E[\| \nabla f(x_{t+1})- \nabla f(x_t) + \nabla f(x_t) -g_t\|^2]
		+ \frac{1}{\eta^2} E[\| x_{t+1} - x_t\|^2] \\
		&-\frac{2}{\eta} \langle\nabla f(x_{t+1})-\nabla f(x_t), x_{t+1}-x_t \rangle\\
		&\leq 2E[\| \nabla f(x_{t+1})- \nabla f(x_t)\|^2] + 2E[\|\nabla f(x_t) -g_t\|^2]
		+ \frac{1}{\eta^2}E[\| x_{t+1} - x_t\|^2] \\
		&+\frac{2}{\eta}\|\nabla f(x_{t+1})-\nabla f(x_t)\|\| x_{t+1}-x_t \| \\
		&\leq 2 \tilde{L}^2 E[\| x_{t+1}- x_t\|^2] + 2E[\|\nabla f(x_t) -g_t\|^2]+ \frac{1}{\eta^2}E[\| x_{t+1} - x_t\|^2] +\frac{2 \tilde{L}}{\eta}\| x_{t+1}-x_t \|^2 \\
		& = 2E[\|\nabla f(x_t) -g_t\|^2]
		+ \frac{1+2\tilde{L}\eta + 2\tilde{L}^2\eta^2}{\eta^2}\| x_{t+1}-x_t \|^2 
		\end{align*}
		the second inequality holds due to $\|a+b\|^2\leq 2\|a\|^2 + 2\|b\|^2$ and $-\langle a,b \rangle \leq  \|a\|\|b\|$, the third inequality holds due to $\tilde{L}-$smoothness. 
		
		Put in the distance bound for $\| x_{t+1}-x_t \|^2$ in Lemma 4.1, we get 
		\begin{align*}
		& E[ dist( 0, \hat{\partial}F( x_{t+1}))^2] \\
		& \leq  2E[\|\nabla f(x_t) -g_t\|^2]
		+ \frac{1+2\tilde{L}\eta + 2\tilde{L}^2\eta^2}{\eta^2}\frac{2\eta}{1-2\tilde{L}\eta}(F\left( {x}_{t}\right)-F\left( {x}_{t+1}\right))\\
		&+\frac{1+2\tilde{L}\eta + 2\tilde{L}^2\eta^2}{\eta^2} \frac{\eta}{\tilde{L}-2 \tilde{L}^2\eta}\left\|{g}_{t}-\nabla f\left( {x}_{t}\right)\right\|^{2}\\
		&\leq (\frac{1+ 4\tilde{L}\eta - 2\tilde{L}^2\eta^2}{\tilde{L}\eta-2\tilde{L}^2\eta^2}) E[\|\nabla f(x_t) -g_t\|^2]
		+(\frac{2+4\tilde{L}\eta + 4\tilde{L}^2\eta^2}{\eta - 2\tilde{L}\eta^2}) ( F( x_t ) - F( x_{t+1}))\\
		&= C_1 E[\|\nabla f(x_t) -g_t\|^2]  + C_2 ( F( x_t ) - F( x_{t+1}))
		\end{align*}
		where $ 0 < \eta < \frac{1}{2\tilde{L}}$, 
		let $C_1= \frac{1+ 4\tilde{L}\eta - 2\tilde{L}^2\eta^2}{\tilde{L}\eta-2\tilde{L}^2\eta^2}$,
		$C_2 = \frac{2+4\tilde{L}\eta + 4\tilde{L}^2\eta^2}{\eta - 2\tilde{L}\eta^2}$, $\Delta = F( x_1 ) - F( x^*)$, where $x^*$ is the optimal of Problem (1).
		
		Therefore, we get the final result, 
		\begin{align*}
		E[dist( 0, \hat{\partial}F( x_R))^2] 
		\leq \frac{C_1}{T} \sum_{t=1}^T E[\|\nabla f(x_t) -g_t\|^2]  + \frac{C_2}{T}\Delta.
		\end{align*}
	\end{proof}
	
	\subsection{ Mini-batch ProxSGD with uniform sampling}
	
	\noindent\textbf{Corollary 4.2.1.}
	Considering Problem (1) under Assumption 1, if $ 0 < \eta < \frac{1}{2 \tilde{L}}$, then for all $T \geq 1$, ProxSGD (Algorithm 1) with uniform sampling will have,
	\begin{align*}
	E[& dist( 0, \hat{\partial}F( x_T))^2]  \leq \frac{1}{b} \frac{1}{n} \frac{n-b}{n-1}(\sum_{i=1}^n G_i^2) C_1 +  \frac{C_2}{T} \Delta,
	\end{align*}
	where $C_1= \frac{1+ 4\tilde{L}\eta - 2\tilde{L}^2\eta^2}{\tilde{L}\eta-2\tilde{L}^2\eta^2}$,
	$C_2 = \frac{2+4\tilde{L}\eta + 4\tilde{L}^2\eta^2}{\eta - 2\tilde{L}\eta^2}$, $\Delta = F( x_1 ) - F( x^*)$.   
	\begin{proof}
		For uniform sampling,
		$$v_i=\frac{n-b}{n-1}$$ and $$p_i = \frac{b}{n}$$
		\begin{align*}
		\frac{1}{T} \sum_{t=1}^T E[\|\nabla f(x_t) -g_t\|^2] &\Longleftrightarrow \frac{1}{n^2}\sum_{i=1}^n \frac{v_i}{p_i}\|\nabla f_i(x_t)\|^2\\
		&\leq \frac{1}{n^2}\sum_{i=1}^n \frac{n-b}{n-1} \frac{n}{b} G_i^2
		\end{align*}
		we can get the desired result.
	\end{proof}
	
	\noindent\textbf{Corollary 4.2.2.}
	Considering Problem (1) under Assumption 1, if $ 0 < \eta < \frac{1}{2 \tilde{L}}$,  $T = \frac{2C_2 \Delta}{\epsilon^2} $ and a fixed batchsize $b = \frac{(2\sum_{i=1}^n G_i^2) C_1}{n \epsilon^2}$, ProxSGD (Algorithm 1) with uniform sampling will have,
	\begin{align*}
	\mathrm{E}\left[\operatorname{dist}\left(0, \hat{\partial} F\left(\mathrm{x}_{R}\right)\right)^{2}\right] \leq \epsilon^{2}.
	\end{align*}
	Then, the computational complexity is $$ \frac{4(\sum_{i=1}^n G_i^2) C_1 C_2\Delta}{n \epsilon^4}.$$
	
	\begin{proof}
		\begin{align*}
		\frac{1}{b} \frac{1}{n} \frac{n-b}{n-1}(\sum_{i=1}^n G_i^2) C_1 \leq \frac{\epsilon^2}{2}
		&\overset{\frac{n-b}{n-1}<1}{\Longrightarrow} \frac{1}{b} \frac{1}{n} (\sum_{i=1}^n G_i^2) C_1 \leq \frac{\epsilon^2}{2}\\
		&\Longrightarrow b = \frac{(2\sum_{i=1}^n G_i^2) C_1}{n \epsilon^2}
		\end{align*}
		\begin{align*}
		\frac{C_2}{T} \Delta  \leq \frac{\epsilon^2}{2} \Longrightarrow T \geq \frac{2C_2}{\epsilon^2} \Delta 
		\end{align*}
		Then, we get $\mathrm{E}\left[\operatorname{dist}\left(0, \hat{\partial} F\left(\mathrm{x}_{R}\right)\right)^{2}\right] \leq \epsilon^{2}$. And the total computational complexity is $T * b = \frac{4(\sum_{i=1}^n G_i^2) C_1 C_2\Delta}{n \epsilon^4}$ with  $ 0 < \eta < \frac{1}{2 \tilde{L}}$.
	\end{proof}
	
	\subsection{ Mini-batch ProxSGD with independent sampling}
	\noindent\textbf{Corollary 4.2.3.}
	Considering Problem (1) under Assumption 1, if $ 0 < \eta < \frac{1}{2 \tilde{L}}$, then for all $T \geq 1$, ProxSGD (Algorithm 1) with independent sampling will have,
	\begin{align*}
	E[dist( 0, \hat{\partial}F( x_T))^2] 
	\leq \frac{C_1}{n^2}  \frac{1}{b+k-n}(\sum_{i=1}^{k} G_{i})^2    + \frac{C_2}{T} \Delta,
	\end{align*}
	where $C_1= \frac{1+ 4\tilde{L}\eta - 2\tilde{L}^2\eta^2}{\tilde{L}\eta-2\tilde{L}^2\eta^2}$,
	$C_2 = \frac{2+4\tilde{L}\eta + 4\tilde{L}^2\eta^2}{\eta - 2\tilde{L}\eta^2}$, $\Delta = F( x_1 ) - F( x^*)$.   
	
	\begin{proof}
		\begin{align*}
		E[ dist( 0, \hat{\partial}F( x_T))^2] 
		& \leq   \frac{C_1}{T n^2}\sum_{t=1}^T \sum_{i=1}^n \frac{ v_i }{ p_{i} } G_i^2  + \frac{C_2}{T}\Delta\\
		&=\frac{C_1}{n^2} \sum_{i=1}^n \frac{1-p_i }{ p_{i} } G_i^2  + \frac{C_2}{T}\Delta\\
		&=\frac{C_1}{n^2} \sum_{i=1}^n \frac{ 1}{ p_{i} } G_i^2 - \frac{C_1}{n^2} \sum_{i=1}^n  G_i^2 + \frac{C_2}{T}\Delta\\
		&= \frac{C_1}{n^2}  \left( \frac{1}{b+k-n}\sum_{i=1}^{k} (\sum_{j=1}^{k} G_{j}) G_{i} - \sum_{i=1}^{k} G_i^2 \right)  + \frac{C_2}{T} \Delta\\
		&= \frac{C_1}{n^2}  \left( \frac{1}{b+k-n}(\sum_{i=1}^{k} G_{i})^2 - \sum_{i=1}^{k} G_i^2 \right)  + \frac{C_2}{T} \Delta,\\
		&\leq \frac{C_1}{n^2}  \frac{1}{b+k-n}(\sum_{i=1}^{k} G_{i})^2   + \frac{C_2}{T} \Delta.
		\end{align*}
	\end{proof}
	
	\noindent\textbf{Corollary 4.2.4.}
	Based on the above corollary, we further have that $G_i$ are similar, $k=n$, $T = \frac{2C_2 \Delta}{\epsilon^2} $ and a fixed batchsize $b = \frac{2(\sum_{i=1}^{n} G_{i})^2 C_1}{n^2 \epsilon^2}$, ProxSGD (Algorithm 1) with independent sampling will have,
	\begin{align*}
	\mathrm{E}\left[\operatorname{dist}\left(0, \hat{\partial} F\left(\mathrm{x}_{R}\right)\right)^{2}\right] \leq \epsilon^{2}.
	\end{align*}
	Then, the computational complexity is $$ \frac{4(\sum_{i=1}^{k} G_{i})^2 C_1 C_2\Delta}{n^2 \epsilon^4}.$$
	
	\begin{proof}
		Since $k=n$,
		\begin{align*}
		E[dist( 0, \hat{\partial}F( x_T))^2] 
		\leq \frac{C_1}{n^2}  \frac{1}{b}(\sum_{i=1}^{n} G_{i})^2    + \frac{C_2}{T} \Delta,
		\end{align*}
		\begin{align*}
		\frac{C_1}{n^2}  \frac{1}{b}(\sum_{i=1}^{n} G_{i})^2 \leq \frac{\epsilon^2}{2} \Rightarrow b = \frac{C_1}{n^2}  \frac{2}{ \epsilon^2}(\sum_{i=1}^{n} G_{i})^2 
		\end{align*}
		
		Then, we get $\mathrm{E}\left[\operatorname{dist}\left(0, \hat{\partial} F\left(\mathrm{x}_{R}\right)\right)^{2}\right] \leq \epsilon^{2}$. And the total computational complexity is $T * b = \frac{4(\sum_{i=1}^n G_i^2) C_1 C_2\Delta}{n^2 \epsilon^4}$.
	\end{proof}
	
	\section{ProxSARAH with Arbitrary Sampling}
	We then give the analysis of ProxSARAH with arbitrary sampling (Algorithm 2) under non-smooth non-convex regularized problems. Before proving the main Theorem 5.4, we have some prepared lemmas.
	
	Similar to ProxSGD, the update of $x_{t+1}^{(j)}$ in ProxSARAH is: 
	\begin{align}
	x_{t+1}^{(j)} &\in \operatorname*{argmin}_{x \in \mathbb{R}^d} \bigg\{ r( x ) + \frac{1}{2\eta}\|x - ( x_t^{(j)} - \eta \mathcal{V}_t^{(j)})\|^2 \bigg\} \nonumber \\
	&= \operatorname*{argmin}_{x \in \mathbb{R}^d} \bigg\{ r( x ) + \langle \mathcal{V}_t^{(j)}, x - x_t^{(j)}\rangle +  \frac{1}{2\eta}\|x - x_t^{(j)} \|^2 \bigg\} \label{eq:5}
	\end{align}
	then by the definition of $argmin$, we have
	\begin{align*}
	0 \in \hat{\partial}r( x_{t+1}^{(j)} ) + \mathcal{V}_t^{(j)}  + \frac{1}{\eta}( x_{t+1}^{(j)} - x_t^{(j)}).
	\end{align*}
	Hence,
	\begin{align*}
	- \mathcal{V}_t^{(j)}  - \frac{1}{\eta}( x_{t+1}^{(j)}  - x_t^{(j)} )\in \hat{\partial}r( x_{t+1}^{(j)} ), 
	\end{align*}
	implying 
	\begin{align}\label{eq:6}
	\nabla f( x_{t+1}^{(j)}) - \mathcal{V}_t^{(j)}  - \frac{1}{\eta}( x_{t+1}^{(j)}  - x_t^{(j)})
	&\in \nabla f( x_{t+1}^{(j)} )+ \hat{\partial}r( x_{t+1}^{(j)}) \nonumber\\
	&=  \hat{\partial}F( x_{t+1}^{(j)}  ).
	\end{align}
	
	\begin{lemma}\label{lemma:v_f}
		Considering updating formula in SARAH: $\mathcal{V}_{t}^{(j)} = \sum_{i \in S_t^{(j)} } \frac{ 1}{ n p_{i} }  (\nabla  f_{i}(x_{t}^{(j)}) - \nabla f_{i}( x_{t-1}^{(j)} ) ) + \mathcal{V}_{t-1}^{(j)}$, then $\forall 1\leq t\leq m$, $j \geq 1$, we have 
		\begin{align*} 
		E[\|\mathcal{V}_t^{(j)} -  \nabla f( x_{t}^{(j)}) \|^2]  &= \sum_{k=1}^t E[\|\mathcal{V}_k^{(j)}-\mathcal{V}_{k-1}^{(j)}\|^2] 
		- \sum_{k=1}^t E[\|\nabla f( x_{k}^{(j)})-\nabla f( x_{k-1}^{(j)})\|^2].
		\end{align*}
	\end{lemma}	
	
	\begin{proof}
		Since  $\mathcal{V}_{t}^{(j)} = \sum_{i \in S_t^{(j)} } \frac{ 1}{ n p_{i} }  (\nabla  f_{i}(x_{t}^{(j)}) - \nabla f_{i}( x_{t-1}^{(j)} ) ) + \mathcal{V}_{t-1}^{(j)}$,
		then $$E[\mathcal{V}_{t}^{(j)}-\mathcal{V}_{t-1}^{(j)}] =E[\sum_{i \in S_t^{(j)} } \frac{ 1}{ n p_{i} }  (\nabla  f_{i}(x_{t}^{(j)}) - \nabla f_{i}( x_{t-1}^{(j)} ) )]  = \nabla f( x_{t}^{(j)}) - \nabla f( x_{t-1}^{(j)}).$$
		Let's further bound $E[\|\mathcal{V}_t^{(j)} -  \nabla f( x_{t}^{(j)}) \|^2]$,
		\begin{align*}
		E&[\|\mathcal{V}_t^{(j)} -  \nabla f( x_{t}^{(j)}) \|^2]  
		= E[\|\mathcal{V}_t^{(j)} - \mathcal{V}_{t-1}^{(j)} + \mathcal{V}_{t-1}^{(j)} -  \nabla f( x_{t-1}^{(j)}) +  \nabla f( x_{t-1}^{(j)})  -  \nabla f( x_{t}^{(j)})\|^2] \\
		&=E[\|\mathcal{V}_t^{(j)} - \mathcal{V}_{t-1}^{(j)}\|^2] +E[\| \mathcal{V}_{t-1}^{(j)} -  \nabla f( x_{t-1}^{(j)})\|^2] +  E[\|\nabla f( x_{t-1}^{(j)})  -  \nabla f( x_{t}^{(j)})\|^2] \\
		&+ 2E[\langle \mathcal{V}_t^{(j)} - \mathcal{V}_{t-1}^{(j)}, v_{t-1}^{(j)} -  \nabla f( x_{t-1}^{(j)})\rangle] \\
		&+ 2E[\langle \mathcal{V}_t^{(j)} - \mathcal{V}_{t-1}^{(j)}, \nabla f( x_{t-1}^{(j)})  -  \nabla f( x_{t}^{(j)}) \rangle] \\
		&+ 2E[\langle \mathcal{V}_{t-1}^{(j)} -  \nabla f( x_{t-1}^{(j)}),\nabla f( x_{t-1}^{(j)})  -  \nabla f( x_{t}^{(j)}) \rangle] \\
		&= E[\|\mathcal{V}_t^{(j)} - \mathcal{V}_{t-1}^{(j)}\|^2] +E[\| \mathcal{V}_{t-1}^{(j)} -  \nabla f( x_{t-1}^{(j)})\|^2] +  E[\|\nabla f( x_{t-1}^{(j)})  -  \nabla f( x_{t}^{(j)})\|^2] \\
		& - 2E[\|\nabla f( x_{t-1}^{(j)})  -  \nabla f( x_{t}^{(j)}) \|^2] \\
		&=E[\|\mathcal{V}_t^{(j)} - \mathcal{V}_{t-1}^{(j)}\|^2] +E[\| \mathcal{V}_{t-1}^{(j)} -  \nabla f( x_{t-1}^{(j)})\|^2] -  E[\|\nabla f( x_{t-1}^{(j)})  -  \nabla f( x_{t}^{(j)})\|^2] \\
		&=\sum_{k=1}^t E[\|\mathcal{V}_k^{(j)}-v_{k-1}^{(j)}\|^2] - \sum_{k=1}^t E[\|\nabla f( x_{k}^{(j)})-\nabla f( x_{k-1}^{(j)})\|^2].
		\end{align*}
	\end{proof}
	
	\begin{lemma}\label{lemma:Q}
		Considering updating formula in SARAH: $\mathcal{V}_{t}^{(j)} = \sum_{i \in S_t^{(j)} } \frac{ 1}{ n p_{i} }  (\nabla  f_{i}(x_{t}^{(j)}) - \nabla f_{i}( x_{t-1}^{(j)} ) ) + \mathcal{V}_{t-1}^{(j)}$,  then $\forall 1\leq t\leq m$, $j \geq 1$, we further have 
		\begin{align*} 
		E[\|\mathcal{V}_t^{(j)} -  \nabla f( x_{t}^{(j)}) \|^2] \leq Q \sum_{k=1}^{t} E[\|x_{k}^{(j)}-x_{k-1}^{(j)}\|^2],
		\end{align*}
		where $Q=\sum_{i=1}^{n} \frac{v_iL_i^2}{p_i n^2}$. 
	\end{lemma}
	
	\begin{proof}
		Let $\xi_i = \frac{1}{np_i} (\nabla f_i(x_k^{(j)})-\nabla f_i(x_{k-1}^{(j)}))$,
		then 
		\begin{align*}
		&E[\|\mathcal{V}_k^{(j)}-\mathcal{V}_{k-1}^{(j)}\|^2]-\|\nabla f( x_{k}^{(j)})-\nabla f(x_{k-1}^{(j)})\|^2\\
		&= E[\|\sum_{i \in S^{(j)} } \frac{ 1}{ n p_{i} }  (\nabla  f_{i}(x_{k}^{(j)}) - \nabla f_{i}( x_{k-1}^{(j)})\|^2 ] 
		- \|\frac{1}{n} \sum_{i=1}^n[\nabla f_i( x_{k}^{(j)})-\nabla f_i(x_{k-1}^{(j)})]\|^2\\
		&\overset{Lemma~D.1 }{=}E[\|\sum_{i \in S^{(j)} } \xi_i\|^2 ] - \|\frac{1}{n} \sum_{i=1}^n\xi_i \|^2\\
		&\overset{Lemma~D.1 }{\leq}\sum_{i=1}^n v_i p_i \|\xi_i\|^2\\
		&= \sum_{i=1}^n v_i p_i \|\frac{1}{np_i} (\nabla f_i(x_k^{(j)})-\nabla f_i(x_{k-1}^{(j)})) \|^2\\
		&\leq \sum_{i=1}^n \frac{v_i}{n^2 p_i} L_i^2 \| x_k^{(j)}-x_{k-1}^{(j)} \|^2 \overset{\Delta}{=} Q \| x_k^{(j)}-x_{k-1}^{(j)} \|^2
		\end{align*}
		The last inequality holds due to $L_i$-smoothness assumption. 
		Then we use Lemma \ref{lemma:v_f},
		\begin{align*}
		E[\|\mathcal{V}_t^{(j)} -  \nabla f( x_{t}^{(j)}) \|^2]
		&= \sum_{k=1}^t E[\|\mathcal{V}_k^{(j)}-\mathcal{V}_{k-1}^{(j)}\|^2 - \|\nabla f( x_{k}^{(j)})-\nabla f( x_{k-1}^{(j)})\|^2]\\
		&\leq Q \sum_{k=1}^t E[\| x_k^{(j)}-x_{k-1}^{(j)} \|^2].
		\end{align*}
	\end{proof}
	
	\begin{lemma} \label{lemma:F}
		\begin{align*} 
		\langle \mathcal{V}_t^{(j)}  - \nabla f( x_t^{(j)}  ), x_{t+1}^{(j)}  - x_t^{(j)} \rangle +  \frac{1}{2}(\frac{1}{\eta} -\tilde{L} )\|x_{t+1}^{(j)}  - x_t^{(j)}  \|^2 
		\leq F( x_t^{(j)}  ) - F( x_{t+1}^{(j)} ).
		\end{align*}
	\end{lemma}
	
	\begin{proof}
		From updating rule (\ref{eq:5}), we also have 
		\begin{align}
		r( x_{t+1}^{(j)}  ) + \langle \mathcal{V}_t^{(j)}, x_{t+1}^{(j)}  - x_t^{(j)} \rangle +  \frac{1}{2\eta}\|x_{t+1}^{(j)}  - x_t^{(j)}  \|^2 
		\leq r( x_t^{(j)} ).
		\end{align}
		By Lemma 2.1 that $f( x )$ is $\tilde{L}$-smoothness, we further obtain that
		\begin{align} \label{ineq:f2}
		f( x_{t+1}^{(j)} ) \leq f( x_t^{(j)} ) + \langle \nabla f( x_t^{(j)} ), x_{t+1}^{(j)} -x_t^{(j)} \rangle 
		+ \frac{\tilde{L}}{2}\|x_{t+1}^{(j)} -x_t^{(j)} \|^2.
		\end{align}
		Combining these two inequalities, we obtain the result.
	\end{proof}
	
	\begin{lemma} \label{lemma:x}
		\begin{align*}
		\sum_{t=1}^m  E[\|x_{t+1}^{(j)}-x_{t}^{(j)}\|^2] 
		\leq  \frac{1}{ \frac{1}{2}(\frac{1}{\eta} -2\tilde{L} ) - \frac{mQ}{2\tilde{L}}} E[ F( x_0^{(j)}  )  - F( x_{m+1}^{(j)} ) ], 
		\end{align*}
		where $\frac{1}{2}(\frac{1}{\eta} -2\tilde{L} ) - \frac{mQ}{2\tilde{L}} > 0$.
	\end{lemma}
	
	\begin{proof}
		By Lemma \ref{lemma:F}, we obtain that 
		\begin{align*} 
		F( x_{t+1}^{(j)} )  -F( x_t^{(j)}  )  &\leq
		-\langle \mathcal{V}_t^{(j)}  - \nabla f( x_t^{(j)}  ), x_{t+1}^{(j)}  - x_t^{(j)} \rangle -  \frac{1}{2}(\frac{1}{\eta} -\tilde{L} )\|x_{t+1}^{(j)}  - x_t^{(j)}  \|^2 \\
		&\leq\frac{1}{2\tilde{L}}\|\mathcal{V}_t^{(j)}  - \nabla f( x_t^{(j)}  )\|^2 + \frac{\tilde{L}}{2}\|x_{t+1}^{(j)}  - x_t^{(j)}\|^2 - \frac{1}{2}(\frac{1}{\eta} -\tilde{L} )\|x_{t+1}^{(j)}  - x_t^{(j)}  \|^2 \\
		&=\frac{1}{2\tilde{L}}\|\mathcal{V}_t^{(j)}  - \nabla f( x_t^{(j)}  )\|^2 - \frac{1}{2}(\frac{1}{\eta} -2\tilde{L} )\|x_{t+1}^{(j)}  - x_t^{(j)}  \|^2. 
		\end{align*}
		The second inequality holds due to $-\langle a,b \rangle \leq \frac{1}{2}\|a\|^2+ \frac{1}{2}\|b\|^2$.
		
		Taking expectation on both sides and plugging Lemma \ref{lemma:Q} in, we get:
		\begin{align*} 
		E[ F( x_{t+1}^{(j)} )  -F( x_t^{(j)}  ) ] 
		&\leq\frac{1}{2\tilde{L}}E[\|\mathcal{V}_t^{(j)}  - \nabla f( x_t^{(j)}  )\|^2] - \frac{1}{2}(\frac{1}{\eta} -2\tilde{L} )E[\|x_{t+1}^{(j)}  - x_t^{(j)}  \|^2] \\
		&\leq  \frac{Q}{2\tilde{L}} \sum_{k=1}^{t} E[\|x_{k}^{(j)}-x_{k-1}^{(j)}\|^2] - \frac{1}{2}(\frac{1}{\eta} -2\tilde{L} )E[\|x_{t+1}^{(j)}  - x_t^{(j)}  \|^2].
		\end{align*}
		Re-adjust $m$ and add one more term, we get
		\begin{align*} 
		E[ F( x_{m+1}^{(j)} )  -F( x_1^{(j)}  ) ] 
		&\leq  \frac{Q}{2\tilde{L}} \sum_{t=1}^m \sum_{k=1}^{t} E[\|x_{k}^{(j)}-x_{k-1}^{(j)}\|^2] - \frac{1}{2}(\frac{1}{\eta} -2\tilde{L} )\sum_{t=1}^{m} E[\|x_{t+1}^{(j)}  - x_t^{(j)}  \|^2]\\
		&\leq  \frac{Q}{2\tilde{L}} \sum_{t=1}^m \sum_{k=1}^{t} E[\|x_{k+1}^{(j)}-x_{k}^{(j)}\|^2] - \frac{1}{2}(\frac{1}{\eta} -2\tilde{L} )\sum_{t=1}^{m} E[\|x_{t+1}^{(j)}  - x_t^{(j)}  \|^2]\\
		& \leq  \frac{mQ}{2\tilde{L}} \sum_{t=1}^m  E[\|x_{t+1}^{(j)}-x_{t}^{(j)}\|^2] - \frac{1}{2}(\frac{1}{\eta} -2\tilde{L} )\sum_{t=1}^{m} E[\|x_{t+1}^{(j)}  - x_t^{(j)}  \|^2]\\
		&=(\frac{mQ}{2\tilde{L}}- \frac{1}{2}(\frac{1}{\eta} -2\tilde{L} ) ) \sum_{t=1}^m  E[\|x_{t+1}^{(j)}-x_{t}^{(j)}\|^2]
		\end{align*}
		By noticing that $x_0^{(j)}=x_1^{(j)}$, we get
		\begin{align*} 
		\sum_{t=1}^m  E[\|x_{t+1}^{(j)}-x_{t}^{(j)}\|^2] \leq 
		\frac{1}{ \frac{1}{2}(\frac{1}{\eta} -2\tilde{L} ) - \frac{mQ}{2\tilde{L}}} E[ F( x_0^{(j)}  )  - F( x_{m+1}^{(j)} ) ], 
		\end{align*}
		where $\frac{1}{2}(\frac{1}{\eta} -2\tilde{L} ) - \frac{mQ}{2\tilde{L}} > 0$.
	\end{proof}
	
	\noindent\textbf{Theorem 5.1.}
	Considering Problem (1) under Assumption 1,  then $\forall 1\leq t\leq m$, $j \geq 1$, ProxSARAH-AS (Algorithm 2) will have,
	\begin{align*}
	\frac{1}{m\mathcal{J}}&\sum_{j=1}^{\mathcal{J}}\sum_{t=1}^m E[dist( 0, \hat{\partial}F( x_{t+1}^{(j)}))^2]   
	\leq  
	\frac{1}{m\mathcal{J}}(24 \tilde{L} 
	+\frac{4mQ}{\tilde{L}} ) E[F( \tilde{x}^1  ) - F( x^*)], 
	\end{align*}
	where $Q=\sum_{i=1}^{n} \frac{v_iL_i^2}{p_i n^2}$.
	
	\begin{proof}
		Let's try to bound $dist( 0, \hat{\partial}F( x_{t+1}^{(j)}  ))$. By Eq (\ref{eq:6}),
		\begin{align*}
		dist( 0, \hat{\partial}F( x_{t+1}^{(j)}  ))^2
		&=  \| \nabla f( x_{t+1}^{(j)}) - \mathcal{V}_t^{(j)}  - \frac{1}{\eta}( x_{t+1}^{(j)}  - x_t^{(j)} )\|^2 \\
		&=  \| \nabla f( x_{t+1}^{(j)}) - \mathcal{V}_t^{(j)}\|^2  +   \frac{1}{\eta^2}\| x_{t+1}^{(j)}  - x_t^{(j)} \|^2 -\frac{2}{\eta}  \langle \nabla f( x_{t+1}^{(j)} ) - \mathcal{V}_t^{(j)}, x_{t+1}^{(j)}  - x_t^{(j)}  \rangle  
		\end{align*}
		Then by reorganizing inequality in Lemma \ref{lemma:F} , we obtain:
		\begin{align*}
		-\langle \nabla f( x_{t}^{(j)} ) - \mathcal{V}_t^{(j)} , x_{t+1}^{(j)}  - x_t^{(j)} \rangle \leq F( x_t^{(j)}  ) - F( x_{t+1}^{(j)}  ) - \frac{1}{2}(\frac{1}{\eta} -\tilde{L} )\|x_{t+1}^{(j)}  - x_t^{(j)} \|^2
		\end{align*}
		i.e.,
		\begin{align*}
		-\langle \nabla f( x_{t+1}^{(j)} ) - \mathcal{V}_t^{(j)},  x_{t+1}^{(j)}  - x_t^{(j)} \rangle 
		&\leq F( x_t^{(j)}  ) - F( x_{t+1}^{(j)})-  \frac{1}{2}(\frac{1}{\eta} -\tilde{L} )\|x_{t+1}^{(j)}  - x_t^{(j)}  \|^2 \\ 
		&-\langle \nabla f( x_{t+1}^{(j)} )  - \nabla f( x_t^{(j)}  ), x_{t+1}^{(j)}  - x_t^{(j)} \rangle.
		\end{align*}
		Plug in the above result,
		\begin{align*}
		&dist( 0, \hat{\partial}F( x_{t+1}^{(j)}  ))^2 \\
		&\leq \| \nabla f( x_{t+1}^{(j)}) - \mathcal{V}_t^{(j)}\|^2  +  \frac{\tilde{L}}{\eta}\| x_{t+1}^{(j)}  - x_t^{(j)} \|^2 
		+\frac{2}{\eta} (F( x_t^{(j)}) - F(x_{t+1}^{(j)}))\\
		& - \frac{2}{\eta}\langle \nabla f( x_{t+1}^{(j)} )  - \nabla f( x_t^{(j)}  ), x_{t+1}^{(j)}  - x_t^{(j)} \rangle  \\
		&\leq \| \nabla f( x_{t+1}^{(j)}) - \mathcal{V}_t^{(j)}\|^2  +  \frac{\tilde{L}}{\eta}\| x_{t+1}^{(j)}  - x_t^{(j)} \|^2 
		+\frac{2}{\eta} (F( x_t^{(j)}) -F(x_{t+1}^{(j)}))\\
		&+\frac{2}{\eta} \|\nabla f( x_{t+1}^{(j)} )  - \nabla f( x_t^{(j)}  )\| \| x_{t+1}^{(j)}  - x_t^{(j)} \| \\
		&\leq \| \nabla f( x_{t+1}^{(j)}) - \mathcal{V}_t^{(j)}\|^2  +  \frac{\tilde{L}}{\eta}\| x_{t+1}^{(j)}  - x_t^{(j)} \|^2 +\frac{2}{\eta} ( F( x_t^{(j)}  ) - F( x_{t+1}^{(j)} ) )+\frac{2\tilde{L}}{\eta}\| x_{t+1}^{(j)}  - x_t^{(j)} \|^2  \\
		&=\|\nabla f( x_{t+1}^{(j)}) - \mathcal{V}_t^{(j)}\|^2  +  \frac{3\tilde{L}}{\eta}\| x_{t+1}^{(j)}  - x_t^{(j)} \|^2 + \frac{2}{\eta} ( F(x_t^{(j)}) - F( x_{t+1}^{(j)})) \\ 
		&\leq 2\| \nabla f( x_{t}^{(j)}) - \mathcal{V}_t^{(j)}\|^2 + 2\| \nabla f( x_{t+1}^{(j)}) - \nabla f( x_{t}^{(j)})\|^2  +  \frac{3\tilde{L}}{\eta}\| x_{t+1}^{(j)}  - x_t^{(j)} \|^2+\frac{2}{\eta} ( F( x_t^{(j)}  ) - F( x_{t+1}^{(j)}) ) \\ 
		&\leq 2\|\mathcal{V}_t^{(j)} -  \nabla f( x_{t}^{(j)}) \|^2  +  ( 2 \tilde{L}^2 +\frac{3\tilde{L}}{\eta})\| x_{t+1}^{(j)}  - x_t^{(j)} \|^2 +\frac{2}{\eta} ( F( x_t^{(j)}  ) - F( x_{t+1}^{(j)})), 
		\end{align*}
		where the third inequality and the last inequality are due to $\|\nabla f(x_{t+1}^{(j)})-\nabla f(x_t^{(j)})\|\leq \tilde{L}\|x_{t+1}^{(j)}-x_t^{(j)}\|$,  penultimate inequality is due to Young's inequality $\|a+b\|^2 \leq 2\|a\|^2 +2\|b\|^2$.
		
		By summing over $t = 1, ..., m$, using the result in Lemma \ref{lemma:Q}, and taking the expectation, 
		
		\begin{align*}
		&\sum_{t=1}^m E[dist( 0, \hat{\partial}F( x_{t+1}^{(j)}))^2]\\  
		&\leq  (2Qm + 2\tilde{L}^2+\frac{3\tilde{L}}{\eta})
		\sum_{t=1}^m E[\| x_{t+1}^{(j)}  - x_t^{(j)} \|^2] +\frac{2}{\eta}  E[F( x_1^{(j)}  ) - F( x_{m+1}^{(j)})]\\
		&=  (2Qm + 2\tilde{L}^2+\frac{3\tilde{L}}{\eta})
		\sum_{t=1}^m E[\| x_{t+1}^{(j)}  - x_t^{(j)} \|^2] +\frac{2}{\eta}  E[F( x_0^{(j)}  ) - F( x_{m+1}^{(j)})]
		\end{align*}
		
		Next, plugging in  Lemma \ref{lemma:x} for the term $\sum_{t=1}^m E[\| x_{t+1}^{(j)}  - x_t^{(j)} \|^2] $,
		
		\begin{align*}
		&\sum_{t=1}^m E[dist( 0, \hat{\partial}F( x_{t+1}^{(j)}))^2]\\   
		& \leq  \frac{2Qm + 2\tilde{L}^2+\frac{3\tilde{L}}{\eta}}{ \frac{1}{2}(\frac{1}{\eta} -2\tilde{L} ) - \frac{mQ}{2\tilde{L}}} E[ F( x_0^{(j)}  )  - F( x_{m+1}^{(j)} ) ]
		+\frac{2}{\eta}  E[F( x_0^{(j)}  ) - F( x_{m+1}^{(j)})]  \\
		&= (\frac{2Qm + 2\tilde{L}^2+\frac{3\tilde{L}}{\eta}}{ \frac{1}{2}(\frac{1}{\eta} -2\tilde{L} ) - \frac{mQ}{2\tilde{L}}} 
		+\frac{2}{\eta} ) E[F( x_0^{(j)}  ) - F( x_{m+1}^{(j)})]. 
		\end{align*}
		
		Therefore:
		\begin{align*}
		\frac{1}{m\mathcal{J}}&\sum_{j=1}^{\mathcal{J}}\sum_{t=1}^m E[dist( 0, \hat{\partial}F( x_{t+1}^{(j)}))^2]  \leq  
		\frac{1}{m\mathcal{J}}(\frac{2Qm + 2\tilde{L}^2+\frac{3\tilde{L}}{\eta}}{ \frac{1}{2}(\frac{1}{\eta} -2\tilde{L} ) - \frac{mQ}{2\tilde{L}}} 
		+\frac{2}{\eta} ) E[F( \tilde{x}^1  ) - F( x^*)]. 
		\end{align*}
		Let $\frac{1}{2}\frac{1}{\eta} = 2\tilde{L}  + \frac{mQ}{\tilde{L}}$, then $\frac{1}{2}\frac{1}{\eta} > \tilde{L}  + \frac{mQ}{\tilde{L}}$ and we get
		\begin{align*}
		\frac{1}{m\mathcal{J}}\sum_{j=1}^{\mathcal{J}}\sum_{t=1}^m E[dist( 0, \hat{\partial}F( x_{t+1}^{(j)}))^2]    
		&=  
		\frac{1}{m\mathcal{J}}(\frac{(2Qm + 2\tilde{L}^2+\frac{3\tilde{L}}{\eta})}{ \frac{1}{4}\frac{1}{\eta} } 
		+\frac{2}{\eta} ) E[F( \tilde{x}^1  ) - F( x^*)] \\
		&\leq  
		\frac{1}{m\mathcal{J}}(16 \tilde{L} 
		+\frac{2}{\eta} ) E[F( \tilde{x}^1  ) - F( x^*)] \\
		&=\frac{1}{m\mathcal{J}}(24 \tilde{L} 
		+\frac{4mQ}{\tilde{L}} ) E[F( \tilde{x}^1  ) - F( x^*)] 
		%&=\frac{1}{m\mathcal{J}}(24 \tilde{L} 
		%+\frac{m\tilde{L}\alpha}{b} ) E[F( \tilde{x}^1  ) - F( x^*)], 
		\end{align*}
		where $Q=\sum_{i=1}^{n} \frac{v_iL_i^2}{p_i n^2}$.
	\end{proof}
	
	\subsection{ProxSARAH with uniform sampling}
	
	\noindent\textbf{Corollary 5.1.2.}
	In order to have $E[dist( 0, \hat{\partial}F( x_R))] \leq \epsilon$, the number of 
	\begin{align}
	\mathcal{J} &= \frac{1}{m\epsilon^2}(24 \tilde{L} 
	+\frac{4m}{\tilde{L}} \frac{1}{b} \frac{1}{n} \frac{n-b}{n-1}(\sum_{i=1}^n L_i^2) )  E[F( \tilde{x}^1  ) - F( x^*)]
	\end{align}
	epochs is required. 
	The computational complexity (number of IFO calls) is 
	\begin{align}
	& \frac{(n+mb)}{m\epsilon^2}(24 \tilde{L} 
	+\frac{4m}{\tilde{L}} \frac{1}{b} \frac{1}{n} \frac{n-b}{n-1}(\sum_{i=1}^n L_i^2) )  E[F( \tilde{x}^1  ) - F( x^*)].
	\end{align}
	
	\begin{proof}
		With uniform sampling, $v_i = \frac{n-b}{n-1}$ and $p_i = \frac{b}{n}$ and further let it less or equal to $\epsilon^2$, we ge the desired result.
	\end{proof}
	
	\subsection{ProxSARAH with independent sampling}
	\noindent\textbf{Corollary 5.1.3.}
	In order to have $E[dist( 0, \hat{\partial}F( x_R))] \leq \epsilon$, the number of 
	\begin{align*}
	\mathcal{J} = \frac{1}{m\epsilon^2}(24 \tilde{L} 
	+\frac{4m}{\tilde{L}}  \frac{C_1}{n^2}  ( \frac{1}{b+k-n}(\sum_{i=1}^{k}  L_{i})^2 - \sum_{i=1}^{k} L_i^2 ) ) E[F( \tilde{x}^1  ) - F( x^*)]
	\end{align*}
	epochs is required. 
	The computational complexity (number of IFO calls) is 
	\begin{align*}
	\frac{(n+mb)}{m\epsilon^2}(24 \tilde{L} 
	+\frac{4m}{\tilde{L}}  \frac{C_1}{n^2}  \frac{1}{b+k-n}(\sum_{i=1}^{k}  L_{i})^2 - \sum_{i=1}^{k} L_i^2 ) ) E[F( \tilde{x}^1  ) - F( x^*)].
	\end{align*}
	
	%\begin{proof}
	%????????????????? \textcolor{red}{If no proof needed, remove corollary.}
	%\end{proof}
	
	\section{ProxSPIDER with Arbitrary Sampling}
	Besides ProxSARAH, we are also interested in studying ProxSPIDER method.
	Similar to ProxSARAH, the following formula  still holds for  ProxSPIDER: 
	\begin{align*}
	\nabla f( x_{t+1}^{(j)}) - \mathcal{V}_t^{(j)}  - \frac{1}{\eta}( x_{t+1}^{(j)}  - x_t^{(j)})
	&\in \nabla f( x_{t+1}^{(j)} )+ \hat{\partial}r( x_{t+1}^{(j)}) \\
	&=  \hat{\partial}F( x_{t+1}^{(j)}  ).
	\end{align*}
	
	\begin{lemma} \label{lemma:v_f2}
		Considering updating formula in SPIDER: $\mathcal{V}_{t}^{(j)} = \sum_{i \in S_t^{(j)} } \frac{ 1}{ n p_{i} }  (\nabla  f_{i}(x_{t}^{(j)}) - \nabla f_{i}( x_{t-1}^{(j)} ) ) + \mathcal{V}_{t-1}^{(j)}$, then $\forall 1\leq t\leq m$, $j \geq 1$, we have 
		\begin{align*} 
		E[\|\mathcal{V}_t^{(j)} -  \nabla f( x_{t}^{(j)}) \|^2]  &= \sum_{k=1}^t E[\|\mathcal{V}_k^{(j)}-\mathcal{V}_{k-1}^{(j)}\|^2] 
		- \sum_{k=1}^t E[\|\nabla f( x_{k}^{(j)})-\nabla f( x_{k-1}^{(j)})\|^2]\\
		&+ E[\|\mathcal{V}_{0}^{(j)} -\nabla f( x_{0}^{(j)})\|^2]  .
		\end{align*}
	\end{lemma}
	
	\begin{proof}
		Since  $\mathcal{V}_{t}^{(j)} = \sum_{i \in S_t^{(j)} } \frac{ 1}{ n p_{i} }  (\nabla  f_{i}(x_{t}^{(j)}) - \nabla f_{i}( x_{t-1}^{(j)} ) ) + \mathcal{V}_{t-1}^{(j)}$,
		then $$E[\mathcal{V}_{t}^{(j)}-\mathcal{V}_{t-1}^{(j)}] =E[\sum_{i \in S_t^{(j)} } \frac{ 1}{ n p_{i} }  (\nabla  f_{i}(x_{t}^{(j)}) - \nabla f_{i}( x_{t-1}^{(j)} ) )]  = \nabla f( x_{t}^{(j)}) - \nabla f( x_{t-1}^{(j)}).$$
		
		Let's further bound $E[\|\mathcal{V}_t^{(j)} -  \nabla f( x_{t}^{(j)}) \|^2]$, notice that the full gradient calculated in SPIDER is a bit different from SARAH,
		
		\begin{align*}
		E&[\|\mathcal{V}_t^{(j)} -  \nabla f( x_{t}^{(j)}) \|^2]  
		= E[\|\mathcal{V}_t^{(j)} - \mathcal{V}_{t-1}^{(j)} + \mathcal{V}_{t-1}^{(j)} -  \nabla f( x_{t-1}^{(j)}) +  \nabla f( x_{t-1}^{(j)})  -  \nabla f( x_{t}^{(j)})\|^2] \\
		&=E[\|\mathcal{V}_t^{(j)} - \mathcal{V}_{t-1}^{(j)}\|^2] +E[\| \mathcal{V}_{t-1}^{(j)} -  \nabla f( x_{t-1}^{(j)})\|^2] +  E[\|\nabla f( x_{t-1}^{(j)})  -  \nabla f( x_{t}^{(j)})\|^2] \\
		&+ 2E[\langle \mathcal{V}_t^{(j)} - \mathcal{V}_{t-1}^{(j)}, \mathcal{V}_{t-1}^{(j)} -  \nabla f( x_{t-1}^{(j)})\rangle] \\
		&+ 2E[\langle \mathcal{V}_t^{(j)} - \mathcal{V}_{t-1}^{(j)}, \nabla f( x_{t-1}^{(j)})  -  \nabla f( x_{t}^{(j)}) \rangle] \\
		&+ 2E[\langle \mathcal{V}_{t-1}^{(j)} -  \nabla f( x_{t-1}^{(j)}),\nabla f( x_{t-1}^{(j)})  -  \nabla f( x_{t}^{(j)}) \rangle] \\
		&= E[\|\mathcal{V}_t^{(j)} - \mathcal{V}_{t-1}^{(j)}\|^2] +E[\| \mathcal{V}_{t-1}^{(j)} -  \nabla f( x_{t-1}^{(j)})\|^2] +  E[\|\nabla f( x_{t-1}^{(j)})  -  \nabla f( x_{t}^{(j)})\|^2] \\
		& - 2E[\|\nabla f( x_{t-1}^{(j)})  -  \nabla f( x_{t}^{(j)}) \|^2] \\
		&=E[\|\mathcal{V}_t^{(j)} - \mathcal{V}_{t-1}^{(j)}\|^2] +E[\| \mathcal{V}_{t-1}^{(j)} -  \nabla f( x_{t-1}^{(j)})\|^2] -  E[\|\nabla f( x_{t-1}^{(j)})  -  \nabla f( x_{t}^{(j)})\|^2] \\
		&=\sum_{k=1}^t E[\|\mathcal{V}_k^{(j)}-\mathcal{V}_{k-1}^{(j)}\|^2] - \sum_{k=1}^t E[\|\nabla f( x_{k}^{(j)})-\nabla f( x_{k-1}^{(j)})\|^2] + E[\|\mathcal{V}_{0}^{(j)} -\nabla f( x_{0}^{(j)})\|^2].
		\end{align*}
	\end{proof}
	
	\begin{lemma}\label{lemma:Q2}
		\begin{align*} 
		E[\|\mathcal{V}_t^{(j)} -  \nabla f( x_{t}^{(j)}) \|^2] &\leq Q \sum_{k=1}^{t} E[\|x_{k}^{(j)}-x_{k-1}^{(j)}\|^2]+ Q'.
		\end{align*}
		where $Q=\sum_{i=1}^{n} \frac{v_iL_i^2}{p_i n^2}$ and $Q' = \sum_{i=1}^n \frac{ v'_i G_i^2 }{ p'_{i} n^2 }  $. 
	\end{lemma}
	
	\begin{proof}
		Follow the same line of proof with Lemma \ref{lemma:Q}, but have one more last term $E[\|\mathcal{V}_{0}^{(j)} -\nabla f( x_{0}^{(j)})\|^2] $, which can be bounded correspondingly.
	\end{proof}

	\begin{lemma} \label{lemma:x2}
		\begin{align*} 
		\sum_{t=1}^m  E[\|x_{t+1}^{(j)}-x_{t}^{(j)}\|^2] \leq 
		\frac{1}{ \frac{1}{2}(\frac{1}{\eta} -2\tilde{L} ) - \frac{mQ}{2\tilde{L}}} E[ F( x_0^{(j)}  )  - F( x_{m+1}^{(j)} ) ], 
		%- \frac{ mQ'}{ \frac{1}{2}(\frac{1}{\eta} -2\tilde{L} ) - \frac{mQ}{2\tilde{L}}},
		\end{align*}
		where $\frac{1}{2}(\frac{1}{\eta} -2\tilde{L} ) - \frac{mQ}{2\tilde{L}} > 0$.
	\end{lemma}
	
	\begin{proof}
		For ProxSPIDER, Lemma \ref{lemma:F} still holds, and we can obtain that 
		\begin{align*} 
		F( x_{t+1}^{(j)} )  -F( x_t^{(j)}  )  &\leq
		-\langle \mathcal{V}_t^{(j)}  - \nabla f( x_t^{(j)}  ), x_{t+1}^{(j)}  - x_t^{(j)} \rangle -  \frac{1}{2}(\frac{1}{\eta} -\tilde{L} )\|x_{t+1}^{(j)}  - x_t^{(j)}  \|^2 \\
		&\leq\frac{1}{2\tilde{L}}\|\mathcal{V}_t^{(j)}  - \nabla f( x_t^{(j)}  )\|^2 + \frac{\tilde{L}}{2}\|x_{t+1}^{(j)}  - x_t^{(j)}\|^2 - \frac{1}{2}(\frac{1}{\eta} -\tilde{L} )\|x_{t+1}^{(j)}  - x_t^{(j)}  \|^2 \\
		&=\frac{1}{2\tilde{L}}\|\mathcal{V}_t^{(j)}  - \nabla f( x_t^{(j)}  )\|^2 - \frac{1}{2}(\frac{1}{\eta} -2\tilde{L} )\|x_{t+1}^{(j)}  - x_t^{(j)}  \|^2 . \\
		\end{align*}
		Taking expectation on both sides and plugging Lemma \ref{lemma:Q2} in, we get:
		\begin{align*} 
		E[ F( x_{t+1}^{(j)} )  -F( x_t^{(j)}  ) ] 
		&\leq\frac{1}{2\tilde{L}}E[\|\mathcal{V}_t^{(j)}  - \nabla f( x_t^{(j)}  )\|^2] - \frac{1}{2}(\frac{1}{\eta} -2\tilde{L} )E[\|x_{t+1}^{(j)}  - x_t^{(j)}  \|^2] \\
		&\leq  \frac{Q}{2\tilde{L}} \sum_{k=1}^{t} E[\|x_{k}^{(j)}-x_{k-1}^{(j)}\|^2] - \frac{1}{2}(\frac{1}{\eta} -2\tilde{L} )E[\|x_{t+1}^{(j)}  - x_t^{(j)}  \|^2] +  \frac{Q'}{2 \tilde{L}}.
		\end{align*}
		Re-adjust $m$ and add one more term, we get
		\begin{align*} 
		E[ F( x_{m+1}^{(j)} )  -F( x_1^{(j)}  ) ] 
		&\leq  \frac{Q}{2\tilde{L}} \sum_{t=1}^m \sum_{k=1}^{t} E[\|x_{k}^{(j)}-x_{k-1}^{(j)}\|^2] - \frac{1}{2}(\frac{1}{\eta} -2\tilde{L} )\sum_{t=1}^{m} E[\|x_{t+1}^{(j)}  - x_t^{(j)}  \|^2]+ \frac{mQ'}{2 \tilde{L}}\\
		&\leq  \frac{Q}{2\tilde{L}} \sum_{t=1}^m \sum_{k=1}^{t} E[\|x_{k+1}^{(j)}-x_{k}^{(j)}\|^2] - \frac{1}{2}(\frac{1}{\eta} -2\tilde{L} )\sum_{t=1}^{m} E[\|x_{t+1}^{(j)}  - x_t^{(j)}  \|^2]+ \frac{mQ'}{2 \tilde{L}}\\
		& \leq  \frac{mQ}{2\tilde{L}} \sum_{t=1}^m  E[\|x_{t+1}^{(j)}-x_{t}^{(j)}\|^2] - \frac{1}{2}(\frac{1}{\eta} -2\tilde{L} )\sum_{t=1}^{m} E[\|x_{t+1}^{(j)}  - x_t^{(j)}  \|^2]+ \frac{mQ'}{2 \tilde{L}}\\
		&=(\frac{mQ}{2\tilde{L}}- \frac{1}{2}(\frac{1}{\eta} -2\tilde{L} ) ) \sum_{t=1}^m  E[\|x_{t+1}^{(j)}-x_{t}^{(j)}\|^2]+ \frac{mQ'}{2 \tilde{L}}
		\end{align*}
		By noticing that $x_0^{(j)}=x_1^{(j)}$, we get
		\begin{align*} 
		\sum_{t=1}^m  E[\|x_{t+1}^{(j)}-x_{t}^{(j)}\|^2] &\leq 
		\frac{1}{ \frac{1}{2}(\frac{1}{\eta} -2\tilde{L} ) - \frac{mQ}{2\tilde{L}}} E[ F( x_0^{(j)}  )  - F( x_{m+1}^{(j)} ) ] - \frac{ 1} { \frac{1}{2}(\frac{1}{\eta} -2\tilde{L} ) - \frac{mQ}{2\tilde{L}}}\frac{mQ'}{2 \tilde{L}} \\
		&\leq 
		\frac{1}{ \frac{1}{2}(\frac{1}{\eta} -2\tilde{L} ) - \frac{mQ}{2\tilde{L}}} E[ F( x_0^{(j)}  )  - F( x_{m+1}^{(j)} ) ] 
		\end{align*}
		where $\frac{1}{2}(\frac{1}{\eta} -2\tilde{L} ) - \frac{mQ}{2\tilde{L}} > 0$.
	\end{proof}
	
	\noindent \textbf{Theorem 6.3.}
	Considering Problem (1) under Assumption 1, then $\forall 1\leq t\leq m$, $j \geq 1$, ProxSPIDER-AS (Algorithm 3) will have,
	\begin{align*}
	\frac{1}{m\mathcal{J}}\sum_{j=1}^{\mathcal{J}}\sum_{t=1}^m E[dist( 0, \hat{\partial}F( x_{t+1}^{(j)}))^2]  
	\leq  \frac{1}{m\mathcal{J}}(24 \tilde{L} 
	+\frac{4mQ}{\tilde{L}} ) E[F( \tilde{x}^1  ) - F( x^*)]+ 2Q', 
	\end{align*}
	where $Q=\sum_{i=1}^{n} \frac{v_iL_i^2}{p_i n^2}$ and $Q^{'}=\sum_{i=1}^{n} \frac{v_i^{'} G_i^2}{p_i^{'} n^2}.$ 
	
	\begin{proof}
		Let's try to bound $dist( 0, \hat{\partial}F( x_{t+1}^{(j)}  ))$. Similar with the analysis of ProxSARAH,
		\begin{align*}
		dist( 0, \hat{\partial}F( x_{t+1}^{(j)}  ))^2 
		\leq 2\|\mathcal{V}_t^{(j)} -  \nabla f( x_{t}^{(j)}) \|^2  +  ( 2 \tilde{L}^2 +\frac{3\tilde{L}}{\eta})\| x_{t+1}^{(j)}  - x_t^{(j)} \|^2 +\frac{2}{\eta} ( F( x_t^{(j)}  ) - F( x_{t+1}^{(j)})), 
		\end{align*}
		
		By summing over $t = 1, ..., m$, using the result in Lemma \ref{lemma:Q2}, and taking the expectation, 
		\begin{align*}
		& \sum_{t=1}^m E[dist( 0, \hat{\partial}F( x_{t+1}^{(j)}))^2]   \\ 
		& \leq  (2Qm + 2\tilde{L}^2+\frac{3\tilde{L}}{\eta})
		\sum_{t=1}^m E[\| x_{t+1}^{(j)}  - x_t^{(j)} \|^2] +\frac{2}{\eta}  E[F( x_1^{(j)}  ) - F( x_{m+1}^{(j)})] + 2mQ' 
		\end{align*}
		
		Next, plugging in  Lemma \ref{lemma:x2} for the term $\sum_{t=1}^m E[\| x_{t+1}^{(j)}  - x_t^{(j)} \|^2] $,
		
		\begin{align*}
		& \sum_{t=1}^m E[dist( 0, \hat{\partial}F( x_{t+1}^{(j)}))^2]  \\ 
		& \leq  \frac{2Qm + 2\tilde{L}^2+\frac{3\tilde{L}}{\eta}}{ \frac{1}{2}(\frac{1}{\eta} -2\tilde{L} ) - \frac{mQ}{2\tilde{L}}} E[ F( x_0^{(j)}  )  - F( x_{m+1}^{(j)} ) ] 
		%\textcolor{red}{- missing item}
		+\frac{2}{\eta}  E[F( x_1^{(j)}  ) - F( x_{m+1}^{(j)})]  +2mQ' \\
		&= (\frac{2Qm + 2\tilde{L}^2+\frac{3\tilde{L}}{\eta}}{ \frac{1}{2}(\frac{1}{\eta} -2\tilde{L} ) - \frac{mQ}{2\tilde{L}}} 
		+\frac{2}{\eta} ) E[F( x_0^{(j)}  ) - F( x_{m+1}^{(j)})] +2mQ'.
		%\textcolor{red}{- missing item}. 
		\end{align*}
		
		Therefore:
		\begin{align*}
		\frac{1}{m\mathcal{J}}&\sum_{j=1}^{\mathcal{J}}\sum_{t=1}^m E[dist( 0, \hat{\partial}F( x_{t+1}^{(j)}))^2]  \leq  
		\frac{1}{m\mathcal{J}}(\frac{2Qm + 2\tilde{L}^2+\frac{3\tilde{L}}{\eta}}{ \frac{1}{2}(\frac{1}{\eta} -2\tilde{L} ) - \frac{mQ}{2\tilde{L}}} 
		+\frac{2}{\eta} ) E[F( \tilde{x}^1  ) - F( x^*)] + 2Q'. 
		\end{align*}
		Let $\frac{1}{2}\frac{1}{\eta} = 2\tilde{L}  + \frac{mQ}{\tilde{L}}$,then $\frac{1}{2}\frac{1}{\eta} > \tilde{L}  + \frac{mQ}{\tilde{L}}$ and we get
		\begin{align*}
		\frac{1}{m\mathcal{J}}\sum_{j=1}^{\mathcal{J}}\sum_{t=1}^m E[dist( 0, \hat{\partial}F( x_{t+1}^{(j)}))^2]    
		&\leq  
		\frac{1}{m\mathcal{J}}(\frac{(2Qm + 2\tilde{L}^2+\frac{3\tilde{L}}{\eta})}{ \frac{1}{4}\frac{1}{\eta} } 
		+\frac{2}{\eta} ) E[F( \tilde{x}^1  ) - F( x^*)] + 2Q'\\
		&\leq  
		\frac{1}{m\mathcal{J}}(16 \tilde{L} 
		+\frac{2}{\eta} ) E[F( \tilde{x}^1  ) - F( x^*)] + 2Q'\\
		&=\frac{1}{m\mathcal{J}}(24 \tilde{L} 
		+\frac{4mQ}{\tilde{L}} ) E[F( \tilde{x}^1  ) - F( x^*)] + 2Q'.
		\end{align*}
	\end{proof}
	
	\subsection{ProxSPIDER with uniform sampling}
	
	With uniform sampling, we obtain the computational complexity:
	
	\noindent\textbf{Corollary 6.1.1.}
	Considering Problem (1) under Assumption 1 and same setup with Theorem 6.1, $B =\frac{2}{n\epsilon^2}(\sum_{i=1}^n G_i^2) $, $m = b = \sqrt{B}$, the computational complexity to achieve $E[dist( 0, \hat{\partial}F( x_R))] \leq \epsilon$  for uniform sampling of ProxSPIDER is 
	\begin{align*}
	\frac{2B}{m}(24 \tilde{L} 
	+\frac{4m}{\tilde{L}} \frac{1}{b} \frac{1}{n}(\sum_{i=1}^n L_i^2) ) E[F( \tilde{x}^1  ) - F( x^*)] = O( \frac{1}{\epsilon^3} ).
	\end{align*}
	
	\begin{proof}
		\begin{align*}
		\frac{1}{m\mathcal{J}}&\sum_{j=1}^{\mathcal{J}}\sum_{t=1}^m E[dist( 0, \hat{\partial}F( x_{t+1}^{(j)}))^2]  \\
		&\leq  
		\frac{1}{m\mathcal{J}}(24 \tilde{L} 
		+\frac{4m}{\tilde{L}} \frac{1}{b} \frac{1}{n}(\sum_{i=1}^n L_i^2) ) E[F( \tilde{x}^1  ) - F( x^*)] + \frac{1}{B} \frac{1}{n}(\sum_{i=1}^n G_i^2)
		\end{align*}
		Then, 
		\begin{align*}
		\frac{1}{B} \frac{1}{n}(\sum_{i=1}^n G_i^2) \leq \frac{\epsilon^2}{2} \Rightarrow B =\frac{2}{n\epsilon^2}(\sum_{i=1}^n G_i^2)
		\end{align*}
		
		\begin{align*}
		&\frac{1}{m\mathcal{J}}(24 \tilde{L} 
		+\frac{4m}{\tilde{L}} \frac{1}{b} \frac{1}{n}(\sum_{i=1}^n L_i^2) ) E[F( \tilde{x}^1  ) - F( x^*)] \\
		& \Rightarrow \mathcal{J} 
		\geq \frac{2}{m\epsilon^2} (24 \tilde{L} 
		+\frac{4m}{\tilde{L}} \frac{1}{b} \frac{1}{n}(\sum_{i=1}^n L_i^2) ) E[F( \tilde{x}^1) - F(x^*)]
		\end{align*}
		Then, we get the final result by $(B*mb)*\mathcal{J}$.
	\end{proof}

	\subsection{ProxSPIDER with independent sampling}
	
	With independent sampling and further assume there is no significant difference in $G_i$'s and $L_i$'s, we obtain the computational complexity:
	
	\noindent\textbf{Corollary 6.1.2.}
	Considering Problem (1) under Assumption 1 and same setup with Theorem 6.1, $B =\frac{2}{n^2\epsilon^2}(\sum_{i=1}^n G_i)^2 $, $m = b = \sqrt{B}$, the computational complexity to achieve $E[dist( 0, \hat{\partial}F( x_R))] \leq \epsilon$ for independent sampling of ProxSPIDER is 
	\begin{align*}
	\frac{2B}{m}(24 \tilde{L} 
	+\frac{4m}{\tilde{L}} \frac{1}{b} \frac{1}{n^2}(\sum_{i=1}^n L_i)^2 ) E[F( \tilde{x}^1  ) - F( x^*)] = O( \frac{1}{\epsilon^3}).
	\end{align*}
	
	\begin{proof}
		\begin{align*}
		\frac{1}{m\mathcal{J}} &\sum_{j=1}^{\mathcal{J}}\sum_{t=1}^m E[dist( 0, \hat{\partial}F( x_{t+1}^{(j)}))^2]  \\
		&\leq  
		\frac{1}{m\mathcal{J}}(24 \tilde{L} 
		+\frac{m}{\tilde{L}} \frac{1}{n^2}   \frac{1}{b}(\sum_{i=1}^{k} L_{i})^2    ) E[F( \tilde{x}^1  ) - F( x^*)] + \frac{1}{n^2}  \frac{1}{B}(\sum_{i=1}^{k} G_{i})^2.
		\end{align*}
		Similar line of proof, we get the desired results.
	\end{proof}
	
	\section{Technique Lemma}
	For completeness, we include the proof for Lemma D.1 here.
	
	\noindent\textbf{Lemma D.1.} \cite{horvath2019nonconvex} %[13]
	Let $\xi_1, \xi_2, ..., \xi_n$ be a vectors in $\mathbb{R}^d$ and let $\tilde{ \xi} = \frac{1}{n}\sum_{i=1}^{n} \xi_i$. Let $S$ be a proper sampling ( i.e., assume that $p_i = Prob( i\in S) > 0$ for all $i$). Assume that there is  $v\in \mathbb{R}^n$ such that 
	\begin{align} \label{ineq:P_v_2}
	\mathbf{P} - pp^T \preceq Diag( p_1v_1, p_2v_2, ..., p_n v_n).
	\end{align}
	
	Then
	\begin{align}
	&E\bigg[\sum_{i\in S} \frac{\xi_i}{np_i}\bigg] = \tilde{\xi} \label{eq:xi},\\
	&E\bigg[\| \sum_{i\in S} \frac{\xi_i}{np_i} - \tilde{\xi}\|^2\bigg] \leq \frac{1}{n^2} \sum_{i=1}^n \frac{v_i}{p_i}\|\xi_i\|^2, \label{ineq:xi}
	\end{align}
	where the expectation is taken over sampling S. Moreover, inequality (\ref{ineq:P_v_2}) can always be satisfied by 
	{$$ v_i=\left\{\begin{array}{ll}{n( 1-p_i)} & {i\leq k(S)} \\ {0} & {otherwise}\end{array}\right.$$}
	where constant $k(S)=\left|\left\{i \in[n] : p_{i}<1\right\}\right|=\max \left\{i : p_{i}<1\right\}$. More specifically, the standard uniform sampling admits $v_{i}=\frac{n-b}{n-1}$ and the independent sampling admits $v_{i}=1-p_{i}$.
	
	\begin{proof}
		First, let's define indicator functions:
		\begin{align*}
		\mathbb{I}_{i \in S}:=\left\{\begin{array}{ll}{ 1 } & {\text { if } i \in S} \\ {0,} & {\text { if } otherwise}\end{array}\right.
		\end{align*}
		
		\begin{align*}
		\mathbb{I}_{i,j \in S}:=\left\{\begin{array}{ll}{ 1 } & {\text { if } i,j \in S} \\ {0,} & {\text { if } otherwise}\end{array}\right.
		\end{align*}
		Then, we have the expectation:
		\begin{align*}
		E\left[\sum_{i \in S} \frac{\xi_{i}}{n p_{i}}\right]=\mathrm{E}\left[\sum_{i=1}^{n} \frac{\xi_{i}}{n p_{i}} \mathbb{I}_{i \in S}\right]=\sum_{i=1}^{n} \frac{\xi_{i}}{n p_{i}} \mathrm{E}\left[\mathbb{I}_{i \in S}\right]=\frac{1}{n} \sum_{i=1}^{n} \xi_{i}=\bar{\xi}
		\end{align*}
		and the variance:
		\begin{align*}
		E\bigg[\| \sum_{i\in S} \frac{\xi_i}{np_i} - \tilde{\xi}\|^2\bigg] & =  E\bigg[\| \sum_{i\in S} \frac{\xi_i}{np_i}\|^2\bigg] - \|\tilde{\xi}\|^2 \\ 
		& =  E\bigg[\sum_{i,j} \frac{\xi_i^T}{np_i}\frac{\xi_j}{np_j} \mathbb{I}_{i,j \in S}\bigg] - \|\tilde{\xi}\|^2 \\
		& =  \sum_{i,j} p_{i,j}\frac{\xi_i^T}{np_i}\frac{\xi_j}{np_j}- \|\tilde{\xi}\|^2 \\
		& =  \sum_{i,j} p_{i,j}\frac{\xi_i^T}{np_i}\frac{\xi_j}{np_j}- \sum_{i,j} \frac{\xi_i^T}{n}\frac{\xi_j}{n} \\
		& =  \frac{1}{n^2}\sum_{i,j} (p_{i,j} - p_ip_j)\frac{\xi_i^T}{p_i}\frac{\xi_j}{p_j} \\
		& =\frac{1}{n^{2}} e^{\top}\left(\left(\mathbf{P}-p p^{\top}\right) \circ \Xi^T \Xi \right) e,
		% &\leq \frac{1}{n^2} \sum_{i=1}^n \frac{v_i}{p_i}\|\xi_i\|^2
		\end{align*}
		where $e$ is the vector all of ones in $\mathbb{R}^{n}$,  $\Xi = [ \frac{\xi_1}{p_1},\frac{\xi_1}{p_1}, ..., \frac{\xi_n}{p_n}] \in  \mathbb{R}^{d \times n} $ and $\circ$ is element-wise production operator.
		
		Since we assume $\mathbf{P}-p p^{\top} \preceq \operatorname{Diag}(p \circ v)$, we have 
		\begin{align*}
		e^{\top}\left(\left(\mathbf{P}-p p^{\top}\right) \circ \Xi^{\top} \Xi\right) e \leq e^{\top}\left(\operatorname{Diag}(p \circ v) \circ \Xi^{\top} \Xi\right) e = \frac{1}{n^2} \sum_{i=1}^n \frac{v_i}{p_i}\|\xi_i\|^2
		\end{align*}
		Therefore, we have
		\begin{align*}
		E\bigg[\| \sum_{i\in S} \frac{\xi_i}{np_i} - \tilde{\xi}\|^2\bigg] \leq  \frac{1}{n^2} \sum_{i=1}^n \frac{v_i}{p_i}\|\xi_i\|^2
		\end{align*}
		%The existence of $v$ satisfying inequlity (\ref{ineq:P_v}) can be found in research [13].
	\end{proof}

\end{document}